\documentclass{article} %
\usepackage[margin=70pt]{geometry}
\usepackage{times}
\usepackage{amsmath,amsfonts,amssymb,amsthm}
\usepackage{bm}
\usepackage{xcolor}
\definecolor{myblue}{HTML}{0034D9} %
\usepackage[colorlinks=true, citecolor=cyan,linkcolor=blue]{hyperref}
\usepackage{url}
\usepackage{booktabs, multirow}                     %
\usepackage{enumerate, enumitem}                    %
\usepackage{mathtools}                              %
\usepackage[thmtools-compat]{keytheorems}           %
\usepackage[table]{xcolor}
\usepackage{arydshln}%
\usepackage{float}                                  %
\usepackage[vlined, ruled]{algorithm2e}             %
\usepackage{caption, subcaption}                    %
\usepackage{wrapfig}                                %
\usepackage[most]{tcolorbox}	     				%
\usepackage{natbib}

\usepackage[capitalize,nameinlink]{cleveref}
\usepackage{adjustbox}
\newtheorem{assumption}{\textbf{H}\hspace{-3pt}}
\Crefname{assumption}{H\hspace{-2pt}}{H\hspace{-3pt}}

\newtheorem{theorem}{Theorem}
\newtheorem{remark}[theorem]{Remark} 
\newtheorem{lemma}[theorem]{Lemma} 
\newtheorem{proposition}[theorem]{Proposition}
\definecolor{cinnamon}{rgb}{0.82, 0.41, 0.12}
\definecolor{darkcerulean}{rgb}{0.03, 0.15, 0.4}%
\definecolor{ceruleanblue}{rgb}{0.14, 0.16, 0.48}
\definecolor{cornflowerblue}{rgb}{0.0, 0.06, 0.54}
\definecolor{SlateGray}{RGB}{90, 100, 115}
\definecolor{WarmGray}{RGB}{120, 110, 100}

\def\leq{\leqslant}
\def\geq{\geqslant}
\def\ds{\displaystyle}

\def\N{\mathbb{N}}

\def\ds{\displaystyle}
\DeclareMathOperator*{\tr}{Tr}

\DeclareMathOperator*{\argmin}{Argmin}
\DeclareMathOperator*{\uniqueargmin}{argmin}
\DeclareMathOperator*{\argmax}{Argmax}

\def\E{\mathbb{E}}
\def\bbE{\mathbb{E}}
\def\R{\mathbb{R}}

\def\Id{\mathrm{Id}}
\def\A{\mathrm{A}}

\def\N{\mathrm{N}}

\newcommand{\mmse}[1][]{\mathrm{MMSE}_{\sigma_{#1}}}

\def\N{\mathbb{N}}
\def\sigmax{\sigma_\mathrm{max}}

\newcommand{\tx}{{\tilde x}}

\DeclareMathOperator{\proj}{Proj}

\newcommand{\lnorm}[1]{\left\|#1\right\|}
\newcommand{\xstar}[1][]{x^*_{\sigma_{#1}}}
\newcommand{\dottedxstar}[1][]{{\dot x}^*_{\sigma_{#1}}}
\newcommand{\NN}[2][0]{\mathcal{N}(#1,#2\mathrm{Id})}
\newcommand{\dd}[1]{\mathop{\mathrm{d}{#1}}}

\newcommand{\genmmse}{GAMMA}

\newcommand{\circled}[1]{   %
    \tikz[baseline=(char.base)]{\node[shape=circle,draw,inner sep=0.5pt] (char) {#1};}
}
\newcommand{\stackclap}[2]{\stackrel{\mathclap{#1}}{#2}}

\renewcommand{\comment}[1]{}

 \Crefname{appendix}{Appendix}{Appendices}

\title{Principled MAP estimation for inverse problems: bridging the gap between convergence and performance}

\author{
Alexandre Lagier$^{*,1}$ \quad
Valentine Tosel$^{*,2}$\quad
Anne Gagneux$^{1}$ \quad
Mathurin Massias$^{3}$ \quad
Ségolène Martin$^{3}$
\\[0.8em]
\begin{tabular}{c}
\small
$^{1}$ ENS de Lyon, CNRS, Université Claude Bernard Lyon 1, Inria, \\
\small
LIP UMR 5668, 69342 Lyon Cedex 07, France
\\[0.3em]
\small 
$^{2}$ Univ. Bordeaux, Inria, Bordeaux INP, IMB,\\
\small
UMR 5251, F-33400 Talence, France
\\[0.3em]
\small
$^{3}$ Inria, ENS de Lyon, CNRS, Université Claude Bernard Lyon 1, \\
\small
LIP UMR 5668, 69342 Lyon Cedex 07, France
\\
\end{tabular}
}
\date{}

\newcommand{\includeproblemimagesafhq}[9]{

    \includegraphics[width=0.16\textwidth]{figures/afhq/#1_clean_batch#2_im#3.eps} \hspace{-5mm} &
    \includegraphics[width=0.16\textwidth]{figures/afhq/#1_noisy_batch#2_im#3_pnsr#4.eps} \hspace{-5mm} &
    \includegraphics[width=0.16\textwidth]{figures/afhq/#1_pnp_flow_batch#2_im#3_iterfinal_pnsr#5.eps} \hspace{-5mm} &
    \includegraphics[width=0.16\textwidth]{figures/afhq/#1_approx_pgd_batch#2_im#3_iterfinal_pnsr#6.eps} \hspace{-5mm} &
    \includegraphics[width=0.16\textwidth]{figures/afhq/#1_generic_mmse_average_batch#2_im#3_iterfinal_pnsr#7.eps} \hspace{-5mm} &
    \includegraphics[width=0.16\textwidth]{figures/afhq/#1_generic_mmse_average_random_batch#2_im#3_iterfinal_pnsr#8.eps} \hspace{-5mm} &
\includegraphics[width=0.16\textwidth]{figures/afhq/#1_generic_mmse_average_2_batch#2_im#3_iterfinal_pnsr#9.eps}
\\

    &
    \textcolor{gray}{\footnotesize PSNR: #4} &
    \textcolor{gray}{\footnotesize PSNR: #5} &
    \textcolor{gray}{\footnotesize PSNR: #6} &
    \textcolor{gray}{\footnotesize PSNR: #7} &
    \textcolor{gray}{\footnotesize PSNR: #8} &
    \textcolor{gray}{\footnotesize PSNR: #9}
}

\newcommand{\includeproblemimagescelebapartone}[6]{

    \includegraphics[width=0.6\textwidth]{figures/celeba/#1_clean_batch#2_im#3.pdf}  &
    \includegraphics[width=0.6\textwidth]{figures/celeba/#1_noisy_batch#2_im#3_pnsr#4.pdf}  &
    \ifthenelse{\equal{#5}{}}{
        \includegraphics[width=0.6\textwidth]{figures/celeba/NA.pdf} 
    }{
        \includegraphics[width=0.6\textwidth]{figures/celeba/#1_pnp_diff_batch#2_im#3_pnsr#5.pdf} 
    } &
    \ifthenelse{\equal{#6}{}}{
        \includegraphics[width=0.6\textwidth]{figures/celeba/NA.pdf} 
    }{
        \includegraphics[width=0.6\textwidth]{figures/celeba/#1_prox_pnp_batch#2_im#3_pnsr#6.pdf} 
    }
}
\renewcommand{\includeproblemimagescelebapartone}[6]{DISABLED by MM}

\newcommand{\includepsnrrow}[7]{
    & \textcolor{gray}{\Huge PSNR: #1} &
    \ifthenelse{\equal{#2}{}}{\textcolor{white}{\footnotesize PSNR: #2}}{\textcolor{gray}{\Huge PSNR: #2}} &
    \ifthenelse{\equal{#3}{}}{\textcolor{white}{\footnotesize PSNR: #2}}{\textcolor{gray}{\Huge PSNR: #3}} &
    \textcolor{gray}{\Huge PSNR: #4} &
    \textcolor{gray}{\Huge PSNR: #5} &
    \textcolor{gray}{\Huge PSNR: #6} &
    \textcolor{gray}{\Huge PSNR: #7}
}

\begin{document}

\maketitle
\let\thefootnote\relax\footnotetext{$^{*}$Equal contribution.}

\begin{abstract}
    Pretrained denoisers provide a powerful way to incorporate image priors into restoration algorithms. Plug-and-Play and RED approaches exploit fixed-noise-level denoisers within first-order optimization schemes, with convergence guarantees, but often struggle to achieve high-quality reconstruction on severely ill-posed inverse problems. In contrast, recent state-of-the-art approaches leverage denoisers derived from flow- or diffusion-based generative models and evaluate them along a sequence of decreasing noise levels. While these methods achieve strong empirical performance, their convergence theory remains limited. In this paper, we bridge this gap by specifically designing an algorithm that combines denoisers at decreasing noise levels with a schedule tailored to ensure convergence. From a Bayesian perspective, we prove that our method converges to a \textit{Maximum a Posteriori} (MAP) estimate, under suitable assumptions. Subsequently, we apply our method to various ill-posed inverse problems and show that it surpasses convergent methods while competing with state-of-the-art empirical ones. 
\end{abstract}

\section{Introduction}

An inverse problem in imaging aims to reconstruct a clean image $x^*\in\R^d$ from a degraded measurement $y\in\R^m$.
In general, the degradation process, or \emph{forward model}, can be written as $y = \A(x^*) + \varepsilon$ where $\A:\R^d \to \R^m$ is the degradation operator and $\varepsilon\sim \NN{\sigma_y^2}$ represents additive white Gaussian noise.
A classical approach estimates the clean image $x^*$ by the minimizer $\hat{x}$ of a composite functional: 
\begin{equation}\label{eq:comp_min}
    \hat x \in \underset{x\in\R^d}\argmin\ f(x) + g(x),
\end{equation} 
where $f$ is a \emph{data-fidelity} term, that depends on the observation $y$ and on the degradation, typically $f(x) \propto \lnorm{\A(x)-y}^2$, whereas $g$ is a \emph{regularization} term encoding prior information about the image.
A key question is therefore how to choose the regularizer $g$.

Early works relied on hand-crafted regularizers, such as the $\ell_1$-norm, total-variation, and wavelet sparsity \citep{rudin1992nonlinear,donoho1998minimax,scherzer2009variational}. 
For these choices of $g$, \eqref{eq:comp_min} can easily be solved with classical gradient-descent or proximal algorithms. 
However, the resulting reconstruction $\hat x$ often lacks perceptual quality due to the limited expressiveness of the regularization. 
In order to integrate more complex prior information about clean images, later approaches considered the use of a pretrained denoiser $D_\sigma$ with fixed noise level $\sigma$ within optimization schemes, as an implicit regularizer.
Plug-and-Play (PnP) methods \citep{venkatakrishnan2013plug,chan2016plug} replaced the proximal operator by such a denoiser, which significantly improved perceptual quality of the reconstruction on simple inverse problems such as denoising and deblurring. 
Under certain conditions on the denoiser, PnP methods can be proven to converge, but the limit of the resulting sequence generally cannot be characterized \citep{ryu2019plug,pesquet2021learning,hurault2022gradient,hurault2023convergent}.
Regularization by Denoising (RED, \citealt{romano2017little}) partially bridges the gap between gradient based methods and PnP ones, by using a specific regularizer: $g(x) = \mu x^\top(x - D_\sigma(x))$. 
When $D_\sigma$ is locally homogeneous and has symmetric Jacobian, 
$\nabla g(x) = \mu (x - D_\sigma(x))$, 
and gradient descent on \eqref{eq:comp_min} leads to \Cref{alg:RED};
however these conditions rarely hold for trained denoisers \citep{reehorst2019red}.

\begin{center}
\begin{minipage}[t]{0.473\textwidth}
\RestyleAlgo{ruled}
\begin{algorithm}[H]
\SetInd{0.0em}{0.2em}
\textbf{Initialization:} $x_0 \in \mathbb{R}^d$, $\alpha,\sigma > 0$\\
\For{$k = 0,1,\ldots$}{
$x_{k+1} = x_k - \alpha\nabla f(x_k)
            - \alpha\mu (x_k-D_\sigma(x_k))$
}
\caption{RED}
\label{alg:RED}
\end{algorithm}
\end{minipage}
\hfill
\begin{minipage}[t]{0.52\textwidth}
\RestyleAlgo{ruled}
\begin{algorithm}[H]
\SetInd{0.0em}{0.2em}
\textbf{Initialization:} $x_0 \in \mathbb{R}^d$, $(\alpha_k)_k,(\sigma_k)_k \searrow 0$\\
\For{$k = 0,1,\ldots$}{
$x_{k+1} = x_k - \alpha_k\nabla f(x_k)
            - \alpha_k\mu_k (x_k-D_{\sigma_k}(x_k))$
}
\caption{Annealed RED}
\label{alg:annealed-RED}
\end{algorithm}
\end{minipage}
\end{center}

Despite appealing guarantees, the performance of such methods remains limited on challenging inverse problems such as inpainting or super-resolution. 
To address this limitation, state-of-the-art algorithms use off-the-shelf diffusion- or flow-based generative models \citep{song2020score,ho2020denoising,lipman2023flow,liu2023flow,albergo2023stochasticinterpolant}. 
Most of these algorithms have two main ingredients: (approximate) MMSE denoisers $D_\sigma$ accross all noise levels $\sigma$, along with an annealed noise schedule with decreasing $\sigma$ across iterations \citep{chung2023dps,song2023pseudoinverse,zhang2024flow}.
Among them, we are specifically interested in optimization-based methods which are explicit adaptations of PnP or RED schemes \citep{zhu2023denoising,renaud2024plug,martin2025pnp,pourya2026flower}. 
While these methods achieve impressive results on highly ill-posed inverse problems, their theoretical understanding lags behind: 
since the denoiser changes at each iteration, so does the underlying objective, and analyses developed for classical PnP methods no longer hold. 
Apart from the denoising case $\A = \Id$ \citep{pesme2025map}, the convergence of annealed schemes and the characterization of their limit remain open questions.

Our work belongs to the line of work that seeks to build interpretable and convergent inverse problems solvers based on generative models \citep{zhang2024flow,pesme2025map}. 
Many generative-based methods rely
on empirical design choices, in particular empirical noise schedules.
Instead, we build an algorithm, interpretable as an annealed RED scheme (\Cref{alg:annealed-RED}), to explicitly solve the \emph{maximum a posteriori} (MAP) problem and derive its schedules to guarantee convergence. 
From a Bayesian perspective, the MAP estimator is the most likely image given the observation $y$:
\begin{equation}\label{eq:map_est}
    \hat{x}_{\text{MAP}} \in \underset{x\in\R^d} \argmax\ \log p(x|y) = \underset{x\in\R^d}\argmin\ -\log p(y|x)-\log p(x).
\end{equation}
Under a Gaussian noise model, $-\log p(y |x)= \lnorm{\A(x)-y}^2 / (2 \sigma_y^2)$ up to an additive constant, and \eqref{eq:map_est} is an instance of \eqref{eq:comp_min} with $f(x) = \lnorm{\A(x)-y}^2 / 2$ and $g = -\sigma_y^2\log p$.
While $f$ is explicit, the negative log-prior $-\log p$ is unknown.
\citet{pesme2025map} showed how MMSE denoisers, which diffusion and flow matching models approximate, can be leveraged to solve MAP denoising problems ($\A = \Id$), i.e. to compute the prox of $- \log p$: their MMSE Averaging scheme provably converges.
For inverse problems, they propose Approx-PGD, which uses MMSE Averaging as an inner loop to approximate proximal steps.
However, it requires a growing number of inner iterations and only guarantees convergence of the objective values, which limits its practical use~(\Cref{sec:experiments}).

Our contributions are the following:
\begin{itemize}%
    \item We introduce Generalized Annealed MMSE Averaging (GAMMA), a generic inverse problem solver using MMSE denoisers evaluated at decreasing noise levels. 
    It extends the MMSE Averaging algorithm of \citet{pesme2025map}, which is restricted to MAP denoising. 
    It also connects to existing methods such as annealed version of RED (\Cref{alg:annealed-RED}), %
    annealed-SNORE \citep{renaud2024plug} and PnP-Flow \citep{martin2025pnp}, which do not possess convergence guarantees. 
    \item  Under log-concavity of the distribution $p$ and mild assumptions on the data-fidelity $f$, we prove that the iterates of GAMMA, and of a variant involving renoising of the iterates, converge to a MAP estimate of the inverse problem. 
    When multiple MAP solutions exist, our algorithm is able to guide the iterates towards the one closest to a reference point.  
    \item We show that our method converges faster than Approx-PGD  \citep{pesme2025map}, and is competitive with state-of-the-art annealed methods on CelebA and AFHQ, making it both provably convergent and practical. 
\end{itemize}

\section{Proposed algorithm: Generalized Annealed MMSE Averaging (GAMMA)}
\subsection{Problem \& Background}
We aim to solve the inverse problem \eqref{eq:comp_min} by computing a MAP estimate
\begin{equation}\label{eq:map}
    \hat{x}\in \argmin_{x \in \R^d} f(x) - \tau \log p(x), \tag{MAP}
\end{equation}
where $f(x) := \frac12\lnorm{\A(x) - y}^2$, $p$ denotes the prior distribution over clean images and $\tau := \sigma_y^2 > 0$ denote the noise level of the additive noise. In our work, we assume the forward process is known, while the underlying probability distribution $p$ is not. 
Rather, we suppose we have access to the associated \emph{Minimum Mean Square Error} (MMSE) denoiser at every noise level $\sigma > 0$, defined by
\begin{equation}
    \mmse(z) = \E_{X\sim p, \varepsilon\sim\NN{}}\left[X\,|\, X + \sigma\varepsilon = z\right]\in\R^d\,.
\end{equation}
In practice, MMSE denoisers can be approximated, up to reparameterization, by Flow Matching or diffusion models. Indeed, those methods implicitly learn the score of the smoothed distribution $p_\sigma = p\ast\mathcal{N}(0, \sigma^2\Id)$ which can be related to the MMSE denoiser at level $\sigma > 0$ through Tweedie's formula \citep{efron2011tweedie}:%
\begin{equation}
    \label{eq:tweedie}
    \mmse(z) = z + \sigma^2\nabla\log p_\sigma(z).
\end{equation}
\citet{pesme2025map} introduce MMSE Averaging, an algorithm to solve \eqref{eq:map} in the denoising setting, namely when $f(x) = \frac{1}{2}\lnorm{x - y}^2$; in that case the MAP is simply the proximal operator of $- \tau \log p$.
Its iterations read as follows
\begin{equation}\label{eq:mmseaveraging}
    x_{k + 1} = \alpha_k y + (1 - \alpha_k)\mmse[k](x_k)\,, \hspace{-26mm} 
    \tag{MMSE Averaging}
\end{equation}
with $(\alpha_k)_k$ a weight sequence, $(\sigma_k)_k$ an annealing noise level sequence. 
Plugging in Tweedie's formula \eqref{eq:tweedie}, we can write each iteration as a gradient step:
\begin{equation}
    x_{k + 1} = x_k - \alpha_k\left((x_k - y) - \left(\dfrac{1 - \alpha_k}{\alpha_k}\right)\sigma_k^2\nabla\log p_{\sigma_k}(x_k)\right) \, .
\end{equation}
The gradient step is taken on a smoothed objective function that is better conditioned than the main objective function $F = f - \tau\log p$. Convergence is derived under suitable assumptions, in particular the log-concavity of $p$ and a specific choice for the sequences $(\alpha_k)_k$ and $(\sigma_k)_k$. However, this scheme is restricted to denoising. 
We now introduce an algorithm in the same spirit as~\eqref{eq:mmseaveraging}, but applicable to a much broader class of inverse problems.

\subsection{Method formulation}

We present Generalized Annealed MMSE Averaging (\genmmse{}), a new algorithm based on MMSE denoisers. %
It generalizes MMSE Averaging to any inverse problems (beyond denoising / prox computation), while preserving its convergence guarantees.
Compared to MMSE Averaging, it replaces the fixed term $y$ by a gradient descent step on a regularized version of the datafit $f$.
\begin{tcolorbox}[colback=WarmGray!5,colframe=SlateGray,boxrule=0.5pt,arc=4pt]
    \textbf{Generalized Annealed MMSE Average (\genmmse{}):} 
    Let $(\alpha_k)_k, (\sigma_k)_k$ be two positive annealing sequences. %
    Starting from an arbitrary initialization $x_0\in\R^d$, we define the following iterates
    \begin{equation}
        \label{eq:genmmse}
        x_{k + 1} = \alpha_k(x_k - \nabla f_{\sigma_k}(x_k)) + (1 - \alpha_k)\mmse[k](x_k),
    \end{equation}
    where $f_\sigma(x) := f(x) + \frac{\lambda(\sigma^2)}{2}\lnorm{x - u}^2$ is a regularized data-fidelity term, with an annealing regularization weight $\lambda(\sigma^2)$, and $u\in \R^d$ a reference point that we can choose. 
\end{tcolorbox}

We also introduce a variant of \genmmse{} in which the denoising of the iterate is replaced by the average of its denoised noisy versions. 
Such averaging was previously proposed by \citet{renaud2024plug} and \citet[Remark 3]{martin2025pnp} to improve the empirical performance of their respective algorithms.
In practice, we show in \Cref{fig:ablation_num_eps} that the number of samples $\varepsilon$ used to approximate the expectation can be taken equal to~1 -- leading to a \textit{de facto} stochastic algorithm.

\begin{tcolorbox}[colback=WarmGray!5,colframe=SlateGray,boxrule=0.5pt,arc=4pt]
\textbf{Renoised-\genmmse{}:} 
    With the same notation and an arbitrary intialization $x_0\in\R^d$, we define the following iterates  
    \begin{equation}
        \label{eq:renoised_genmmse}
        \tx_{k + 1} = \alpha_k(\tx_k - \nabla f_{\sigma_k}(\tx_k)) + (1 - \alpha_k)\E_{\varepsilon \sim \NN{}} \left[\mmse[k](\tx_k + \sigma_k\varepsilon)\right],
    \end{equation}
\end{tcolorbox}
Similarly to \citet{pesme2025map}, our update is an average between a gradient step on the data-fidelity and the denoised previous iterate. Both terms hold complementary information, and averaging them allows us to minimize the composite objective function. Using \eqref{eq:tweedie} yields the following key result.

\begin{restatable}{proposition}{smoothedgrad}
    The iterations of \eqref{eq:genmmse} can be written as a gradient descent step when we fix $\left(\frac{1 - \alpha_k}{\alpha_k}\right)\sigma_k^2 = \tau$:
    \begin{equation}
        x_{k + 1} = x_k - \alpha_k\nabla F_{\sigma_k}(x_k), \quad F_{\sigma_k}(x) := f_{\sigma_k}(x) - \tau\log p_{\sigma_k}(x)
    \end{equation}
    Similarly, the renoised iterations \eqref{eq:renoised_genmmse} writes,
    \begin{equation}
        \tx_{k + 1} = \tx_k - \alpha_k\nabla G_{\sigma_k}(\tx),\quad G_{\sigma_k}(x_k):= f_\sigma(x) - \tau \E_{\varepsilon\sim\NN{}}\left[\log p_{\sigma_k}(x+\sigma_k \varepsilon)\right]
    \end{equation}    
\end{restatable}

Rather than directly running gradient descent steps on the MAP objective $f - \tau \log p$, the algorithm runs it on smoothed objectives $F_{\sigma_k}$ or $G_{\sigma_k}$, with $\sigma_k > 0$ decreasing at each iteration: the convex data-fidelity term is replaced by a strongly convex term and the log-prior is replaced by a smoothed log-prior. 
We observe that $F_\sigma$ converges pointwise to the objective function $F = f - \tau \log p$ as $\sigma$ goes to 0:
\begin{equation}
    F_\sigma(x) = f_\sigma(x) - \tau\log p_\sigma(x) \underset{\sigma \to 0}{\longrightarrow} f(x) - \tau\log p(x) = F(x).
\end{equation}
The same pointwise convergence also holds for $G_\sigma$.

\begin{remark}
    We also consider a stochastic version of \eqref{eq:renoised_genmmse}, where we only use one noise sample instead of averaging over infinitely many noise samples. We chose to detail the setting and the analysis of this stochastic scheme in \Cref{ssec:stochastic_mmse}.
\end{remark}

\section{Related works}
Most existing convergence analyses of PnP and RED algorithms consider a fixed noise level and rely on structural assumptions on the denoiser, such as boundedness, non-expansivity, or suitable Jacobian properties \citep{ryu2019plug,reehorst2019red,terris2020building,hurault2023convergent}. Recent generative-model-based methods instead use denoisers at decreasing noise levels \citep{zhu2023denoising,mardani2024variational,martin2025pnp,pourya2026flower}, for which the implicit regularization changes throughout the iterations and the fixed-noise theory does not readily apply. Our construction is specifically designed for this decreasing-noise setting, with a schedule that preserves a variational interpretation and enables convergence analysis.

\paragraph{Link with previous methods.} 
Our method is connected to several existing works.

\begin{itemize}[leftmargin=*,topsep=-2pt]
    \item \textbf{MMSE Averaging.} The iterates \eqref{eq:genmmse} recover MMSE Averaging when $f(x) = \frac{1}{2}\lnorm{x - y}^2$ and $\lambda(\sigma^2) = 0$.
    \item \textbf{RED.} RED \citep{romano2017little}
    performs gradient descent on $f + g$ with $g(x) = \mu
    x^\top(x - D_\sigma(x))$, which gives \Cref{alg:RED} provided $D_\sigma$ is locally homogeneous with a
    symmetric Jacobian \citep{reehorst2019red}. 
    In the GAMMA iterates \eqref{eq:genmmse}, if one sets $\lambda(\sigma^2) = 0$, a simple computation shows we recover an annealed version of RED (\Cref{alg:annealed-RED}) with $\mu_k = \frac{1-\alpha_k}{\alpha_k}$. 
    \item \textbf{SNORE} \citep{renaud2024plug} adds noise to the input of the
    denoiser in RED, and converges to a critical point of a smoothed MAP objective at
    fixed noise level. With MMSE denoisers, its annealed variant is a single-sample
    version of Renoised-GAMMA \eqref{eq:renoised_genmmse} with
    $\lambda(\sigma^2) = 0$. Unlike \genmmse{}, its noise schedule is not designed
    to preserve convergence as $\sigma$ decreases, and no guarantees are currently
    available for it.
    \item \textbf{PnP-Flow.} Up to renaming the iterates, one iteration of PnP-Flow \citep{martin2025pnp} writes:
    \begin{equation}
    \label{eq:pnpflow_iter}
    \begin{cases}
        z_{k + 1} = D_{\sigma_k}(x_{k} + \sigma_k\varepsilon_k)\,, \quad\varepsilon_k \sim \NN{} \, ,\\
        x_{k + 1} = z_k - \alpha_k\nabla f(z_k) \, .
    \end{cases}
    \hspace*{-1cm}
    \tag{PnP-Flow}
    \end{equation}
    When the denoiser is learned optimally to be the MMSE denoiser, one can then leverage~\eqref{eq:tweedie} to interpret an iteration of PnP-FLow as two consecutive gradient steps, rather than as a single gradient step as in \genmmse{}. While it is in general difficult to interpret PnP-Flow as a minimization algorithm, in the denoising case, i.e., $f(x) = \frac{1}{2}\lnorm{x - y}^2$, a single iteration of \eqref{eq:pnpflow_iter} can be written as an iteration of \eqref{eq:mmseaveraging}, in which noise is added at each step to the input of the denoiser (see details in \Cref{ssec:pnpflow}).
    \item \textbf{Approximate Proximal Gradient Descent.} Finally, as MMSE Averaging allows approximate computation of the proximal operator of $-\tau \log p$, \citet[Algorithm 1]{pesme2025map} proposed to use this approximation in a proximal gradient descent scheme for inverse problems with generic data-fidelity $f$.
    Approx-PGD inherits two limitations of inexact proximal algorithms: the inner loop to approximate the proximal operator requires increasingly many iterations\footnote{at iteration $k$, it requires $\mathcal{O}(k^{1 + \eta})$ with $\eta > 0$} making the algorithm computationally expensive, and the convergence results apply to objective values $F(x_k)$ without ensuring the convergence of the iterates themselves. 
\end{itemize}
Conceptually, our framework provides a common perspective on MMSE Averaging, an annealed variant of RED, and SNORE, while also revealing a close connection with PnP-Flow.
In addition to circumventing the limitations of Approx-PGD described above, our approach gives competitive results on imaging inverse problems, hence bridging the gap between the theory developed in \citet{pesme2025map} and algorithms used in practice.

\paragraph{Optimization perspective.}
From an optimization viewpoint, GAMMA is closely related to continuation methods \citep{allgower2012numerical}: it performs gradient steps on smoothed objectives $F_{\sigma_k}$ which converge pointwise to $F$ as $\sigma_k \to 0$. 
Unlike classical continuation, it takes a single step on each smoothing level instead of solving each intermediate problem. 
Our construction combines two regularization mechanisms. 
First, a vanishing Tikhonov term is added to the data-fidelity, a classical technique in convex optimization and inverse problems whose convergence properties have been studied, e.g., by \citet{attouch2024convex}. 
Second, the prior is Gaussian-smoothed inside the logarithm, through $-\log p_\sigma$. This differs from standard Gaussian smoothing \citep{nesterov2017} which would be applied directly to the objective function (hence to $- \log p$). 
This form arises naturally from MMSE denoisers and requires a specific analysis as $\sigma \to 0$.

\section{Convergence Results}
\label{sec:conv}

In this section, we investigate the convergence of \genmmse{} to the MAP estimate. We show that, under suitable assumptions on the problem, iterates from \eqref{eq:genmmse} and \eqref{eq:renoised_genmmse} converge to a MAP estimate i.e a minimum of $F$. 

\subsection{Assumptions on the problem}
We impose some assumptions on the data-fidelity $f$ and $p$. Indeed, we need to ensure that $F$ is well defined and admits minimizers (i.e. MAP estimates exist) in order to prove convergence. Therefore, we enforce the following properties:

\begin{assumption}
    \label{ass:datafit}
    The function $f$ is convex, lower-bounded, twice continuously differentiable and $L_f$-smooth.
\end{assumption}

\begin{assumption}
    \label{ass:prior}
    The probability distribution $p$ on $\R^d$ is such that $p > 0$ and $-\log p$ is convex, three times differentiable and with bounded third derivative. We define $M \geq 0$ as
	\begin{equation}
		M = \underset{x\in\R^d}{\sup}\lnorm{\nabla^3\log p(x)}_F\,,
	\end{equation}
	where for $\A\in\R^{d\times d\times d}$, $\lnorm{\A}_F := \left(\sum_{i,j,k}\A_{i,j,k}^2\right)^{1/2}$ corresponds to the Frobenius norm.
\end{assumption}

\paragraph{Discussion on the assumptions.}
While \Cref{ass:datafit} on the data-fidelity term $f$ is typically satisfied in standard inverse problems, the log-concavity assumption on the prior $p$ in \Cref{ass:prior} is more restrictive. Nevertheless, log-concavity provides a natural and tractable setting for studying decreasing-noise schemes and is a standard simplifying assumption in theoretical analyses of annealed methods \citep{dalalyan2017sampling,brosse2018normalizing,durmus2019analysis}.
It is also the setting adopted by \citet{pesme2025map} to establish convergence guarantees for MMSE Averaging.

These assumptions guarantee some nice properties for $F$ and $F_\sigma$, proven in \Cref{ssec:logpsigma}.

\begin{restatable}{lemma}{coerc}\label{lemm:coerc}
    Under \Cref{ass:datafit} and \Cref{ass:prior}, $F$ is well-defined, convex and coercive.
\end{restatable}
Notably, $F$ admits a non-empty and convex set of minimizer, thus the MAP estimate problem is well-defined.
Additionally, the convexity of $-\log p$ induces the convexity of $-\log p_\sigma$ by the Prékopa-Leindler inequality.  We can then establish the following key regularity properties of $F_\sigma$.

\begin{restatable}{proposition}{conditioning}\label{prop:conditionning}
    For $\sigma > 0$, $F_\sigma$ is $L_\sigma$-smooth and $\mu_\sigma$-strongly convex with
    \begin{equation}
        L_\sigma = L_f + \frac{\tau}{\sigma^2} + \lambda(\sigma^2),\quad\mu_\sigma = \lambda(\sigma^2)\,.
    \end{equation}
\end{restatable}

This result, directly adapted from \citet[Proposition~4]{pesme2025map}, motivates our method: $F_\sigma$ has more favorable optimization properties than $F$, suggesting that minimizing $F_\sigma$ may be more tractable than minimizing $F$ directly. In particular, for any $\sigma>0$, the strong convexity of $F_\sigma$ ensures that $
\xstar := \argmin F_\sigma$ is well-defined and unique. As $\sigma\to0$, however, the strong-convexity parameter $\mu_\sigma$ vanishes, and the behavior of $x_\sigma^\star$ is therefore not immediate. A key step in our analysis is to show that, under suitable assumptions, $(\xstar)_{\sigma >0}$ converges to a minimizer of $F$ as $\sigma\to0$.

\subsection{On the minimizers of \texorpdfstring{$F$ and $F_\sigma$}{F and its smooth versions}}\label{ssec:minimizerpath}

To investigate the properties of the mapping $\sigma^2\mapsto\xstar$, we enforce some mild assumptions on the regularization weight $\lambda$, which plays a key role in our analysis. Specifically, we assume that $\lambda$ is continuously differentiable on $(0,+\infty)$ and satisfies $\lambda(0)=0$. 
The implicit function theorem then ensures that the mapping $\sigma^2\mapsto\xstar$ is continuously differentiable on $(0,+\infty)$. However, its behavior and its regularity as $\sigma\to 0$ are not immediate and are the object of the following propositions.

\begin{restatable}[Limit of the minimizers]{proposition}{limminim}\label{prop:lim_minim}
    Let $\sigmax > 0$. Then, $(\xstar)_{\sigma\in (0, \sigmax]}$ is bounded. More specifically,
    \begin{enumerate}[label=(\roman*)]
        \item as $\sigma \to 0$, every cluster point of $\xstar$ is a minimizer of $F$.
        \item additionally, if $\frac{\sigma^2}{\lambda(\sigma^2)} \underset{\sigma\to 0}\longrightarrow 0$, then $\xstar \underset{\sigma \to 0}{\longrightarrow} x^\dagger$, where 
        \begin{equation}
            x^\dagger := \uniqueargmin \{\|z-u\|\,\vert\, z\in \argmin F\},
        \end{equation}
        which exists and is well defined.
    \end{enumerate}
\end{restatable}
See \Cref{ssec:behavminim} for the proof.
The second statement is useful: if $\lambda(\sigma^2)$ vanishes more slowly than $\sigma^2$ as $\sigma\to0$, then the minimizers $\xstar$ converge to the minimizer of $F$ closest to $u$, i.e., the projection of $u$ onto $\argmin F$. Thus, we have not only convergence of the minimizers to a point of $\argmin F$, but also uniqueness of the limit, which can be determined via the choice of $u\in \R^d$.

We next establish a result controlling the regularity of the path $\sigma \mapsto \xstar$ (see \Cref{ssec:behavminim} for the proof).
\begin{restatable}[Regularity of the minimizers]{proposition}{reglips}\label{prop:reglips}
    Let $\sigmax > 0$. 
    There exist three constants $C_1, C_2, C_3>0$ depending on the problem parameters, such that for any $0 < \sigma_1 \leq \sigma_2 \leq \sigmax$,
    \begin{equation}
        \lnorm{\xstar[1] - \xstar[2]} \leq \left(C_1 + \dfrac{C_3}{\lambda(\sigma_1^2)}\right)(\sigma_2^2 - \sigma_1^2) + C_2\log\left(\dfrac{\lambda(\sigma_2^2)}{\lambda(\sigma_1^2)}\right)
    \end{equation}
\end{restatable}

\subsection{Convergence of the method}\label{ssec:cv}

We now state our main theoretical result: the convergence of \genmmse{} and the different mode of convergence we can obtain according to the schedules chosen.

\begin{restatable}[Convergence of \genmmse{}]{theorem}{iteratescv}\label{thm:iterates_cv}
    Let $(\sigma_k)_k\in \R_+^\N$ be a sequence strictly decreasing to $0$ and, for any $k\geq 0$, let $\lambda_k:=\lambda(\sigma_k^2) \in (0,1)$ and $\alpha_k := \frac{\sigma_k^2}{\sigma_k^2 + \tau}$.
    Assume moreover that: 
    \begin{enumerate}[label=(\roman*)]
        \item $\ds\sum_{k\geq 0} \sigma_k^2\lambda_k = +\infty$;
        \item $\frac{\sigma_k^2-\sigma_{k + 1}^2}{\sigma_{k + 1}^2\lambda_{k+1}^2} \longrightarrow 0$;
        \item $\frac{\log \lambda_k - \log \lambda_{k+1}}{\sigma_{k + 1}^2\lambda_{k+1}}\longrightarrow 0$;
    \end{enumerate}
    Under \Cref{ass:datafit} and \Cref{ass:prior}, the sequence of iterates $(x_k)_{k\geq 0}$ defined in \eqref{eq:genmmse} from any starting point $x_0\in\R^d$, is bounded and is such that each cluster point of $(x_k)_k$ is a minimizer of $F$. In addition, if $\frac{\sigma^2_k}{\lambda_k} \underset{k\to +\infty}{\longrightarrow} 0$, then the sequence $(x_k)_k$ converges to 
    \begin{equation}
        x^\dagger= \uniqueargmin \{\lnorm{z-u}\,\vert\, z\in \argmin F\}= \proj_{\argmin F}(u),
    \end{equation} the nearest minimizer of $F$ to $u$.
\end{restatable}

\begin{remark}\label{rem:conv_schedules} A generic family of sequences $(\alpha_k)_k, (\sigma_k)_k$ and $(\lambda_k)_k$ satisfying the requirements of \Cref{thm:iterates_cv} is not difficult to exhibit. For example, the following choices match all conditions:
        \begin{equation*}
        \lambda_k = \dfrac{\lambda_0}{(k + 1)^\beta}, \quad \sigma_k^2 = \dfrac{\sigma_0^2}{(k + 1)^{\gamma}},\quad \alpha_k=\dfrac{\sigma_k^2}{\sigma_k^2 + \tau},
        \end{equation*}
         where $0<\beta<\gamma$, $0<\gamma + \beta < 1$, $\sigma_0 > 0$ and $\lambda_0\in(0, 1)$.
\end{remark}

\paragraph{Sketch of proof.}
The full proof is in \Cref{app:pf_gamma}.
It relies on first showing that $\lnorm{x_{k+1} - \xstar[k]}$ goes to 0. 
This is achieved by upper bounding it by a function of $\lnorm{x_{k} - \xstar[k]}$ (using a bound of the form $\lnorm{x_{k+1} - \xstar[k]} < \lnorm{x_{k} - \xstar[k]}$ given by the descent lemma, and the bound on $\lnorm{\xstar[k] - \xstar[k-1]}$ given by \Cref{prop:reglips}).
We then show convergence to 0 by applying a classical lemma on real sequences satisfying a perturbed contraction inequality.
This first result allows proving that the cluster points of $(x_k)_k$ and those of $(x_{\sigma_k})_k$ are the same; applying \Cref{prop:lim_minim} then shows that cluster points of $(x_k)_k$ are minimizers of $F$.

We prove a similar result for the variant \eqref{eq:renoised_genmmse}, Renoised-\genmmse{}. Under the same constraints on the schedules, we show convergence to the MAP estimate.
\begin{restatable}[Convergence of Renoised-\genmmse{}]{theorem}{renoisediteratescv}\label{thm:renoised_iterates_cv}
      Let $(\sigma_k)_k\in \R_+^\N$ be a sequence strictly decreasing to $0$ and, for any $k\geq 0$, let $\lambda_k:=\lambda(\sigma_k^2)\in (0,1)$ and $\alpha_k := \frac{\sigma_k^2}{\sigma_k^2 + \tau}$. Under constraints (i), (ii), (iii) of \Cref{thm:iterates_cv} and $\frac{\sigma_k^2}{\lambda_k}\to 0$, the iterates $(\tx_k)$ from \eqref{eq:renoised_genmmse} converges to $x^\dagger$, i.e,
    \begin{equation}
        \lnorm{\tilde x_k - x^\dagger}\underset{k\to+\infty}{\longrightarrow} 0\,.
    \end{equation}
\end{restatable}
\paragraph{Sketch of proof.}
The full proof is given in \Cref{app:pf_renoised}. It again shows that $\lnorm{x_{k+1}-\xstar[k]}\to0$. Since $\xstar[k]$ does not minimize $G_{\sigma_k}$, the descent lemma cannot be applied directly. We instead combine several inequalities to control this distance and establish a contraction between $\lnorm{x_{k+1}-\xstar[k]}$ and $\lnorm{x_k-\xstar[k]}$. The remainder follows the same argument as above.

\section{Numerical Results}
\label{sec:experiments}

\subsection{Convergence Speed on a Gaussian Mixture Prior}

To better assess the computational advantage of \genmmse{} over Approx-PGD, we consider a synthetic problem in which the prior $p$ is an eight-component Gaussian mixture in dimension $d=100$, with means $\mu_i\sim\NN{}$ for $i=1,\ldots,8$. We consider a random inpainting problem, where the degradation operator $\A$ masks $80\%$ of the vector entries. The corresponding objective is $F(x) = \frac{1}{2}\lnorm{\A(x) - y}^2 - \tau\log p(x)$ with $\tau = (0.3)^2$.
Since $F$ can be evaluated explicitly in this setting, we compare the convergence of both methods as a function of computation time. The results are shown in \Cref{fig:toy_example}.

\begin{figure}[t]
    \centering
    \includegraphics[width=0.7\textwidth]{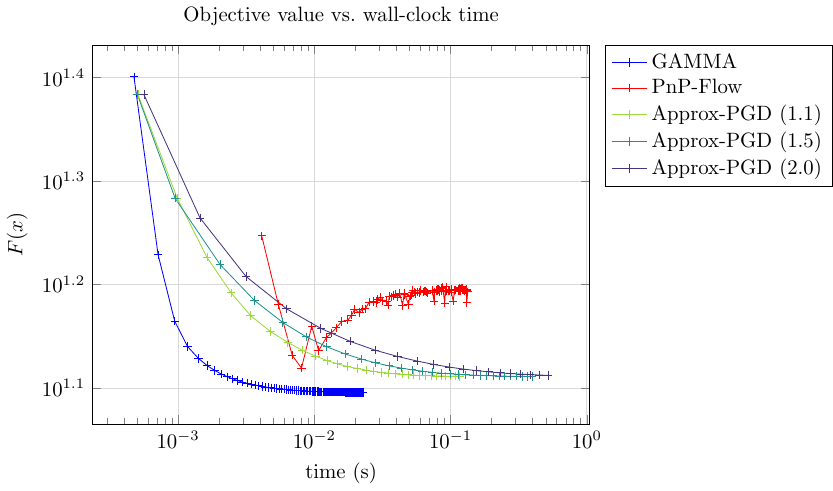}
    \caption{Evolution of the objective value across time of \genmmse{}, PnP-Flow and Approx-PGD $(1+\eta)$ (with $\eta$ controlling the number of sub-iterations). \textbf{PnP-Flow diverges. GAMMA converges faster than Approx-PGD.}
    }
    \label{fig:toy_example}
\end{figure}

\subsection{Benchmarks}

We compare our GAMMA algorithm with two complementary approaches: \textbf{PnP-Flow} \citep{martin2025pnp}, which provides a state-of-the-art reference for image restoration methods with annealed denoisers, and
\textbf{Approx-PGD} \citep{pesme2025map} for the theoretically grounded baseline solving the MAP problem.
These experiments assess whether GAMMA can combine strong empirical performance when used with the identified schedules guaranteeing convergence toward the MAP estimator.
We evaluate three variants of our algorithm:
\begin{itemize}[leftmargin=*,topsep=-2pt]
    \item \textbf{GAMMA} The deterministic scheme of \Cref{eq:genmmse}, combined with the noise schedule identified in \Cref{rem:conv_schedules}, corresponding to the regime for which convergence is theoretically proven.
    \item \textbf{Renoised-GAMMA} The renoised scheme of \Cref{eq:renoised_genmmse}, using the same schedule as above, i.e. also within the theoretically proven regime.
    \item \textbf{Relaxed GAMMA} A variant in which we adopt the noise schedule used by PnP-Flow, corresponding to uniform time sampling in their formulation, i.e. $\sigma_k = \frac{N-k}{N}$ where $N$ is a fixed number of iterations. 
    Note that, for the other sequences $(\lambda_k)$ and $(\alpha_k)$, we remain within the principled scheduling framework of \Cref{rem:conv_schedules}, i.e. $\lambda_k = \lambda_0 \sigma_k$ and $\alpha_k = \sigma_k^2 / (\sigma_k^2+ \tau)$. 
\end{itemize}

We report quantitative results on CelebA-128 \citep{celeba} in \Cref{tab:benchmark_results_celeba}, and on AFHQ-Cat-256~\citep{choi2020starganv2} in \Cref{tab:benchmark_results_afhq} (\Cref{app:expes}, using PSNR, SSIM, and LPIPS (\citealp{zhang2018perceptual}) as evaluation metrics. 
We also provide qualitative comparisons in \Cref{fig:afhq_images}, showing reconstructions obtained by the different methods across several inverse problems.
Training details for the backbone network and hyperparameters are detailed in \Cref{app:expes}.

\paragraph{Discussion}
Approx-PGD, as noted in \citet{vert2026beyond}, tends to produce cartoon-like images, an effect mitigated by early stopping. Despite tuning the number of steps, it does not perform satisfactorily beyond denoising. This is consistent with the original paper \citep{pesme2025map}, which reports no numerical results.
\textbf{GAMMA}, in its noiseless version, performs well on deblurring and random inpainting, substantially improving over Approx-PGD. Yet, on more generative tasks (super-resolution and mask inpainting), it shows a similar collapse toward cartoon-like images, with visible artefacts. 
We attribute it to the absence of early stopping, as we use the same fixed iteration budget for the noiseless and renoised variants.
\textbf{Renoised-GAMMA} is much closer to the PnP-Flow baseline across all tasks, including generative ones: the qualitative results are convincing and LPIPS matches that of PnP-Flow. It remains below PnP-Flow in PSNR, which we potentially attribute to the annealed noise schedule $\sigma_k$ not reaching the observation noise level $\sigma_y=0.05$ within the prescribed iteration budget, so the observation noise is not fully removed.
Finally, the \textbf{relaxed GAMMA} variant shows that, by forcing the noise schedule to reach $0$ within a prescribed number of steps, our algorithm matches the performance of PnP-Flow. 
Overall, these results highlight the versatility of our approach and the importance of carefully designed noise schedules. This contrasts with other renoised annealed RED approaches such as SNORE \citep{renaud2024plug}: although SNORE admits convergence guarantees in the fixed-noise setting, its annealed variant provides neither comparable guarantees nor principled guidance for choosing the noise schedule, and reports no results on challenging generative inverse problems such as super-resolution or box inpainting.
Our experiments also highlight the importance of renoising in generative inverse problems. Our theory incorporates this step, ubiquitous in generative-based methods, while remaining within the MAP framework.

\begin{table}[tbh]
\caption{Comparisons of methods on different inverse problems on the CelebA dataset. Results are averaged across 100 test images. Higher PSNR / SSIM is better, lower LPIPS is better.}
\label{tab:benchmark_results_celeba}
\centering
\setlength{\tabcolsep}{2.8pt}
\renewcommand{\arraystretch}{1.25}
\resizebox{1.0\hsize}{!}{
\begin{tabular}{l c|ccc|ccc|ccc|ccc|ccc}
\toprule
\multirow{3}{*}{Method}
& \multirow{3}{*}{{\small Theory}}
& \multicolumn{3}{c|}{\textcolor{blue}{Denoising}}
& \multicolumn{3}{c|}{\textcolor{blue}{Deblurring}}
& \multicolumn{3}{c|}{\textcolor{blue}{Super-res.}}
& \multicolumn{3}{c|}{\textcolor{blue}{Rand. inpaint.}}
& \multicolumn{3}{c}{\textcolor{blue}{Box inpaint.}} \\

& & \multicolumn{3}{c|}{\textcolor{blue}{\small $\sigma=0.2$}}
& \multicolumn{3}{c|}{\textcolor{blue}{\small $\sigma=0.05$, $\sigma_{\mathrm b}=3.0$}}
& \multicolumn{3}{c|}{\textcolor{blue}{\small $\sigma=0.05$, $\times4$}}
& \multicolumn{3}{c|}{\textcolor{blue}{\small $\sigma=0.01$, $70\%$}}
& \multicolumn{3}{c}{\textcolor{blue}{\small $\sigma=0.05$, $80\times80$}} \\

\cmidrule(lr){3-5}
\cmidrule(lr){6-8}
\cmidrule(lr){9-11}
\cmidrule(lr){12-14}
\cmidrule(lr){15-17}

& & \small PSNR & \small SSIM & \small LPIPS
& \small PSNR & \small SSIM & \small LPIPS
& \small PSNR & \small SSIM & \small LPIPS
& \small PSNR & \small SSIM & \small LPIPS
& \small PSNR & \small SSIM & \small LPIPS \\

\midrule

\rowcolor{gray!10}
Degraded
& &
20.00 & 0.683 & 0.148 &
27.78& 0.736& 0.129&
10.25& 0.174& 0.819&
11.95& 0.188& 1.032&
22.26 & 0.740& 0.212\\

\hdashline

\shortstack[l]{PnP-Flow}
& No&
\textbf{32.74} & \underline{0.915} & 0.056 &
\underline{34.85} & \underline{0.940} & 0.047 &
\underline{32.05} & \underline{0.913} & 0.056&
\textbf{34.86} & \textbf{0.962} & \textbf{0.018} &
\textbf{32.02} & \underline{0.946} & \underline{0.040} \\

\shortstack[l]{Approx-PGD}
& Yes &
26.98 & 0.718 & 0.083 &
23.28 & 0.691 & 0.170 &
15.86 & 0.369 & 0.370 &
16.21 & 0.391 & 0.355 &
6.460 & 0.501 & 0.498 \\

\hdashline

\rowcolor{green!5!violet!10}
\shortstack[l]{GAMMA\\(noiseless)}
& Yes &
30.19 & 0.834 & 0.052 &
26.80 & 0.638 & 0.283 &
18.78 & 0.524 & 0.205 &
19.95 & 0.657 & 0.297 &
24.02 & 0.747 & 0.205  \\

\rowcolor{green!5!violet!10}
\shortstack[l]{GAMMA\\(noise)}
& Yes&
32.57 & 0.905 & \underline{0.049} &
33.59& 0.898& \underline{0.045}& 
31.59& 0.890&  \underline{0.052}& 
32.45& 0.920 & 0.040&
\underline{31.13}& 0.897 & 0.032\\
\rowcolor{green!5!violet!10}
\shortstack[l]{GAMMA\\(relaxed)}
& No &
\underline{32.68} & \textbf{0.919} & \textbf{0.036} &
\textbf{35.27} & \textbf{0.950} & \textbf{0.025} &
\textbf{32.50} & \textbf{0.918} & \textbf{0.049} &
\underline{34.21} & \underline{0.954} & \underline{0.021} &
30.20 & \textbf{0.948} & \textbf{0.026} \\
\bottomrule
\end{tabular}
}
\end{table}

\vspace{-5mm}
\begin{figure}[h!]
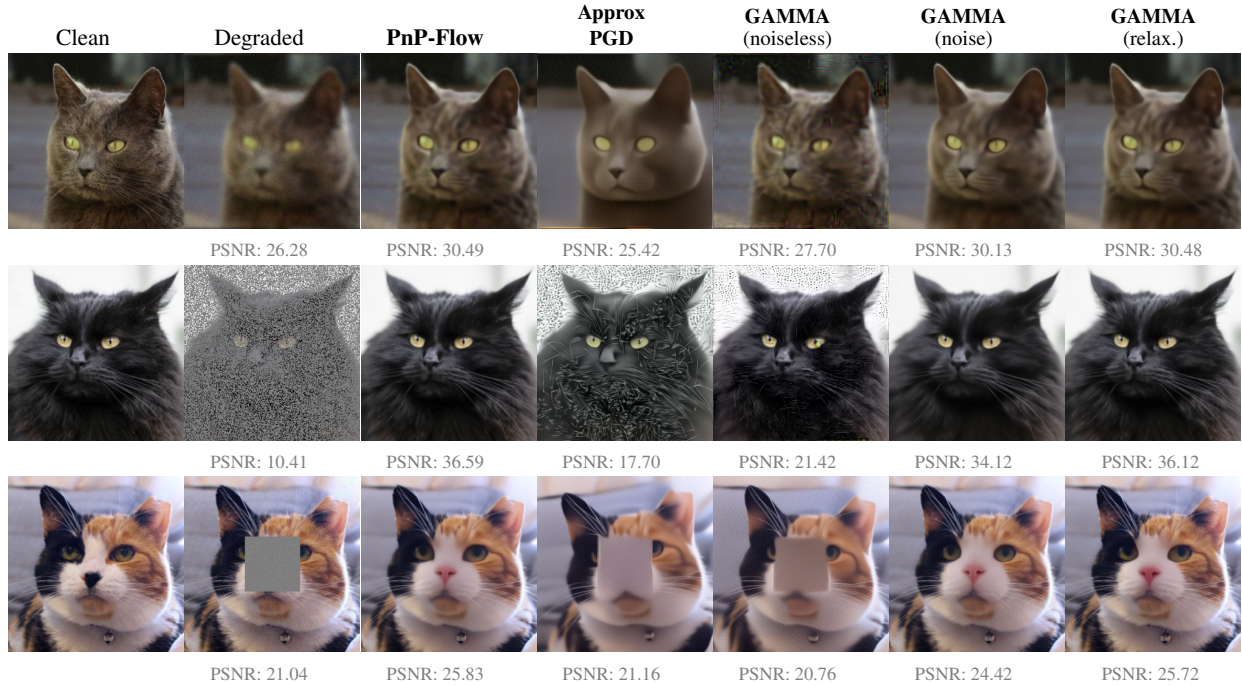

    \centering
    \begin{adjustbox}{max width=1.0\textwidth}
    \begin{tabular}{cccccccc}

    Clean &
    Degraded &
    \textbf{PnP-Flow} &
    \shortstack{\small \textbf{Approx} \\ \small \textbf{PGD}}  &
    \shortstack{\small \textbf{GAMMA} \\ \small (noiseless)}  &
    \shortstack{\small \textbf{GAMMA} \\ \small (noise)} &
    \shortstack{\small \textbf{GAMMA} \\ \small (relax.)}
    \\

    \includeproblemimagesafhq{gaussian_deblurring_FFT}{1}{1}{26.28}{30.49}{25.42}{27.70}{30.13}{30.48}\\
    \includeproblemimagesafhq{random_inpainting}{1}{2}{10.41}{36.59}{17.70}{21.42}{34.12}{36.12} \\
    \includeproblemimagesafhq{inpainting}{1}{3}{21.04}{25.83}{21.16}{20.76}{24.42}{25.72}

    \end{tabular}
    \end{adjustbox}
    \caption{Qualitative results on AFHQ-256. The full images table is in the Appendix, \Cref{fig:afhq_images_full}. }
    \label{fig:afhq_images}
\end{figure}

\section{Conclusion}
We introduced \genmmse{}, a convergent annealed RED-type method, with an optional renoising mechanism, whose noise and regularization schedules are explicitly designed to guarantee convergence to a MAP solution. To the best of our knowledge, this is the first approach to bridge the gap between provably convergent denoiser-based MAP methods, such as \citet{pesme2025map}, and recent generative restoration methods, which typically rely on annealing and renoising without comparable guarantees.
Our experiments show that this gap can be substantially narrowed with the renoised variant, achieving competitive reconstruction quality compared to recent non-convergent generative methods. When the constraints imposed by the theory are relaxed, the method reaches comparable performance, suggesting that the remaining gap is mainly due to the constraints on the current schedules.
Future work will focus on relaxing the log-concavity assumption and extending the convergence analysis of the stochastic renoised scheme to less restrictive noise schedules.

\subsubsection*{Acknowledgments}

This work was granted access to the HPC resources of IDRIS under the allocation 2026AD011017676, \newline 2026-AD011017183, and 2026-AD010616781 made by GENCI.
We gratefully acknowledge the support of the Centre Blaise Pascal’s IT test platform at
ENS de Lyon (Lyon, France) for providing machine learning computing facilities. The platform
operates the SIDUS solution developed by Emmanuel Quemener \citep{quemener2013use}. Ségolène Martin's research benefited from the financial support of GdR IASIS (project: EDOPNP).

\bibliography{refs}
\bibliographystyle{plainnat}

\newpage

\appendix
\crefalias{section}{appendix}
\crefalias{subsection}{appendix}

\subsubsection*{AI use statement}
Generative AI tools were used to assist with writing, language polishing, literature retrieval, and proofreading. The bibliography itself was assembled and checked manually by the authors. AI tools were also used as an additional check of the mathematical derivations and helped identify a few minor errors that did not affect the main results. All proofs were developed independently by the authors, except for \Cref{prop:reglips} for which AI assistance was used both to help formalize the relevant property and during the derivation of its proof. This result was subsequently checked and validated by the authors. The authors take full responsibility for the content of the paper.

\section{Results on the minimizer path \texorpdfstring{$\sigma^2\mapsto\xstar$  (\cref{ssec:minimizerpath})}{}}
In this section, we prove \Cref{lemm:coerc,prop:lim_minim,prop:reglips}. 

\paragraph{Notation}
\begin{itemize}
    \item $\lnorm{\cdot}$ denotes the standard Euclidean norm on $\R^d$,
    \item $\R_+$ denotes $[0,+\infty)$,
    \item for $a\in\R^d$ and $r\geq 0$, $\mathcal{B}\left(a,r\right)$ denotes the closed ball centered in $a$ with radius $r$ (for the standard Euclidean norm),
    \item $\nabla g =\left(\partial_i g\right)_{1\leq i\leq d} \in \R^d$ denotes the gradient of $g:\R^d\to \R$,
    \item $\nabla^2 g =\left(\partial_{i,j} g\right)_{1\leq i,j\leq d} \in \R^{d\times d}$ denotes the Hessian of $g:\R^d\to \R$,
    \item $\nabla^3 g = \left(\partial_{i,j,l} g\right)_{1\leq i,j,l\leq d}\in \R^{d\times d\times d}$ denotes the third derivative tensor of $g:\R^d \to \R$,
    \item $\Delta g = \tr\left(\nabla^2 g\right)=\sum_{i=1}^d\partial_{i,i} g$ denotes the Laplacian of $g$.
\end{itemize}

\subsection{Preliminary results on objective function \texorpdfstring{$F$}{}}
\label[appendix]{ssec:objective_function_res}
\coerc*
\begin{proof}
    Recalling that $F(x) = f(x) - \tau\log p(x)$, $F$ is well-defined since $p>0$ by \Cref{ass:prior}. It is also convex since $f$ is convex (by \Cref{ass:datafit}) and $-\log p$ is convex (by \Cref{ass:prior}). Finally, in order to show that $F$ is coercive, we show that because $p$ is a log-concave positive probability distribution, $-\log p$ is coercive.

    For the sake of contradiction, suppose that $-\log p(x)$ does not tend to infinity as $\lnorm{x}\to +\infty$. Then, there exist $A>0$ and a sequence $(z_k)_{k\in\N}$ in $\R^d$ such that $\|z_k\| \to +\infty$, and  for every $k$, $-\log p(z_k) \leq A$.

    By \Cref{ass:prior}, $-\log p$ is continuous, so there exists $A'\in \R$ such that $-\log p \leq A'$ on the closed unit ball $\mathcal{B}\left(0,1\right)$.
    Then, by convexity of $-\log p$, we deduce that 
    \begin{equation*}
        -\log p(x) \leq \max(A,A'),\qquad \forall x\in \mathcal{C}_k,
    \end{equation*} 
    where 
    \[\mathcal{C}_k = \left\{(1-t)x + t z_k\,  \vert\quad  x\in \mathcal{B}\left(0,1\right), \, t\in[0,1] \right\}.\]
     
    Thus, $p\geq \exp(-\max(A,A'))$ on a sequence of subsets whose mass go to $+\infty$ (since every $\mathcal{C}_k$ contains a cone of volume $\alpha \|z_k\|$, where $\alpha>0$ is independent of $k$). It contradicts the integrability of $p$, and concludes the proof.
\end{proof}
As noted, as an immediate consequence of this lemma, $\argmin F$ is a non-empty, convex, bounded and closed subset of $\R^d$. 

\subsection{Preliminary results \texorpdfstring{on $\log p$ and $\log p_\sigma$}{}}
\label[appendix]{ssec:logpsigma}

We now show that $-\log p_\sigma$ inherits some coercivity properties of $-\log p$. If the coercivity of $-\log p_\sigma$ can be expected as $p_\sigma$ is a smoothed version of $p$, we show a slightly stronger result, by proving that the family $(-\log p_\sigma)_\sigma$ is, in a sense, uniformly coercive.
In what follows, we write sometimes $p_0$ for $p$ to simplify the notation. This convention is natural, since $p_\sigma$ converges pointwise to $p$ when $\sigma \to 0$.   
\begin{proposition}\label{prop:coercunif}
Let $A>0$. Then, there exists $R>0$ such that, for every $\sigma \in [0,+\infty)$,
\[x\in\R^d \text{ and } \|x\|\geq R \Longrightarrow -\log p_\sigma(x)\geq A.\]
\end{proposition}

\begin{proof}
Let $A>0$. By the coercivity of $-\log p$, there exist :
\begin{itemize}
    \item $R>0$ such that $-\log p(x) \geq A+\log(2)$ (i.e $p\leq \frac{1}{2}e^{-A}$) for all $x\in\R^d \setminus \mathcal{B}(0, R)$ , 
    \item since $p$ is also continuous, $m>0$ such that $-\log p$ is lower-bounded by $-\log m$, i.e $p\leq m$ on $\R^d$. 
\end{itemize}
Let $R'>R$. Then, for every $x$ such that $\|x\|\geq R'$,
\begin{align*}
    p_\sigma(x) &= \frac{1}{(2\pi\sigma^2)^{d/2}} \int_{z\in\R^d} p( z)\exp{\left(-\frac{\|x-z\|^2}{2\sigma^2}\right)}\dd{z}\\
    &\leq \frac{1}{2}e^{-A} +  \frac{m}{(2\pi\sigma^2)^{d/2}}  \int_{z\in \mathcal{B}(0, R)}\exp{\left(-\frac{\|x-z\|^2}{2\sigma^2}\right)}\dd{z}.\\
    &\leq \frac{1}{2}e^{-A} +  \frac{m \text{Vol}\left( \mathcal{B}(0, R)\right)}{(2\pi\sigma^2)^{d/2}}  \exp{\left(-\frac{(R'-R)^2}{2\sigma^2}\right)}\,,\\ 
\end{align*}
where $\text{Vol}\left( \mathcal{B}(0, R)\right):= \int_{z\in\mathcal{B}(0,R)} 1\, dz$.
By comparing growth rates and a straightforward derivative calculation, the mapping 
\begin{equation*}
    \sigma^2 \in (0, +\infty) \mapsto \frac{m \text{Vol}\left( \mathcal{B}(0, R)\right)}{(2\pi\sigma^2)^{d/2}}  \exp{\left(-\frac{(R'-R)^2}{2\sigma^2}\right)}
\end{equation*}
admits a continuous extension at $0$, with value $0$, and attains its maximum at $\sigma^2 =\frac{(R'-R)^2}{d}$, where its value is $\frac{d^{d/2}m \text{Vol}\left( \mathcal{B}(0, R)\right)}{(2\pi)^{d/2}(R'-R)^d}  \exp{\left(-\frac{d}{2}\right)}$. Since this quantity tends to $0$ as $R'\to +\infty$, there exists $R''>R$ such that \begin{equation*}
    \frac{d^{d/2}m \text{Vol}\left( \mathcal{B}(0, R)\right)}{(2\pi)^{d/2}(R''-R)^d}  \exp{\left(-\frac{d}{2}\right)} \leq \frac{1}{2}e^{-A}.
\end{equation*}
Thus, for every $\sigma \geq 0$ and $x\in  \R^d \setminus\mathcal{B}(0, R'')$, $p_\sigma (x) \leq e^{-A}$, or equivalently $-\log p_\sigma(x) \geq A$.
This concludes the proof.
\end{proof}
Conversely, for every $x\in \R^d$, the function $\sigma^2 \in [0, +\infty) \mapsto -\log p_\sigma(x)$ is continuous. Hence, it is bounded from above on every compact interval, and in particular on a neighborhood of $0$.

\begin{lemma}\label{lemme:majlog}
Let $x\in\R^d$ and $\sigmax>0$.
There exists $b>0$ such that for every $\sigma\in [0, \sigmax]$, $-\log p_\sigma(x) \leq b$.
\end{lemma}
Finally, we recall a lemma proved in \citet[Lemma 4]{pesme2025map} on the convergence of $-\log p_\sigma$:
\begin{lemma}\label{lem:unifcv}
    On every compact subset of $\R^d$, $(\log p_\sigma)_\sigma$ converges uniformly to $\log p$ as $\sigma\rightarrow 0$.
\end{lemma}
As an immediate consequence, on every compact subset of $\R^d$, $F_\sigma$ converges uniformly to $F$ as $\sigma\to 0$.

\subsection{Behavior in \texorpdfstring{$0$}{}}
\label[appendix]{ssec:behavminim}
We recall that $\lambda : \R_+ \longrightarrow \R_+$ is a strictly increasing function, differentiable on $(0,+\infty)$, and such that $\lambda(0)=0$.
In this subsection, we want to prove the following proposition:
\limminim*
To do so, we state a useful lemma, whose proof is postponed after the proof of \Cref{prop:lim_minim}. 
\begin{lemma}[Asymptotic behavior of $\log p_\sigma$ when $\sigma\to 0$]\label{lemm:eqpsigp}
    There exists a function $\tilde{Q}:\R_+\times \R^d \to \R$ continuous such that, for any $(\sigma, x)\in\R_+\times \R^d$,
    \begin{equation*}
        p_\sigma(x)= p(x)\left(1+\frac{\sigma^2}{2} \Big(\Delta \log p(x)+ \lnorm{\nabla \log p(x)}^2\Big) +\sigma^3 \frac{\tilde{Q}(\sigma, x)}{p(x)}\right).
    \end{equation*}
    As a consequence, for any compact $K\subset \R^d$ and $\sigmax>0$, there exists $\tilde{M}>0$ such that
    \begin{equation*}
    \forall (\sigma, x)\in [0,\sigmax]\times K, \qquad  
        \left|\log\left(\dfrac{p_\sigma(x)}{p(x)}\right)-\frac{\sigma^2}{2}\left(\Delta \log p(x)+\lnorm{\nabla \log p(x)}^2\right)\right|\leq \tilde{M} \sigma^3.
    \end{equation*}
\end{lemma}

This lemma means that, when $\sigma \to 0$, $p_\sigma(x)\simeq p(x)\left(1+\frac{\sigma^2}{2}(\Delta \log p(x)+ \lnorm{\nabla \log p(x)}^2)\right)$, which in fact is a bit more precise than what we need to prove \Cref{prop:lim_minim}. Nevertheless, we state it in this form since it gives a better understanding of the behavior of $\log p_\sigma$.  
\begin{proof}[Proof of \Cref{prop:lim_minim}]
We first prove that the mapping $\sigma^2\mapsto \xstar$ is bounded in a neighborhood of $0$.
Using \Cref{lemme:majlog}, let $b>0$ such that $-\tau\log p_\sigma(u) \leq b$ for every $\sigma \in [0, \sigmax]$. 
Now, thanks to \Cref{prop:coercunif}, let $R>0$ such that for every $\sigma\geq 0$ and for all $x\in \R^d \setminus \mathcal{B}(0,R)$, 
\begin{equation*}
    -\tau\log p_\sigma(x) > f(u) + b - \inf f.
\end{equation*} Then, for all $x\in \R^d \setminus \mathcal{B}(0,R)$ and $\sigma \in [0, \sigmax]$:
\begin{align*}
    F_\sigma(x) &= f(x)- \tau \log p_\sigma(x)+\frac{\lambda(\sigma^2)}{2}\|x-u\|^2 \\
    &\geq \inf f - \tau \log p_\sigma(x) \\
    &> f(u) + b \\
    &\geq f(u) - \tau \log p_\sigma(u) = F_\sigma(u).
\end{align*}
Thus, necessarily, $\xstar \in \mathcal{B}(0,R)$ for every $\sigma \in (0,\sigmax]$ (and the calculation also shows that all the minimizers of $F$ are in $\mathcal{B}(0,R)$).
\begin{enumerate}[label=(\roman*),leftmargin=*]
    \item Let $\bar{x}$ be a cluster point of $\xstar$ as $\sigma^2\to 0$, and let $(\bar{\sigma}_n)_n\in (0,+\infty)^\N$, be a sequence converging to $0$ such that $x_{\bar{\sigma}_n}^* \underset{n\to +\infty}{\longrightarrow} \bar{x}$. Consider also $x^\star$ a minimizer of $F$.
Then, by \Cref{lem:unifcv}, $F_\sigma$ converges uniformly to $F$ on the compact $\mathcal{B}(0,R)$ as $\sigma \to 0$, hence $\underset{n\to \infty}{\lim} F_{\bar{\sigma}_n}(x_{\bar{\sigma}_n}^*)=F(\bar x)$, and:
\begin{equation*}
    F(\bar x) = \underset{n\to \infty}{\lim} F_{\bar{\sigma}_n}(x_{\bar{\sigma}_n}^*) \leq \underset{n\to \infty}{\lim} F_{\bar{\sigma}_n}(x^\star) = F(x^\star) = \inf F.
\end{equation*}
Thus $\bar x \in \argmin F$, and this concludes the proof.
\item  We first justify the existence and uniqueness of $x^\dagger= \argmin \{\|z-u\|\,\vert\, z\in \argmin F\}$. By continuity, coercivity and convexity of $F$ (\Cref{lemm:coerc}, \Cref{ass:prior}), $\argmin F \neq \emptyset$ is compact and convex. By strict convexity, the function $z\mapsto \lnorm{z-u}^2$ attains a unique minimum over $\argmin F$.

Now, we assume that $\frac{\sigma^2}{\lambda(\sigma^2)}\underset{\sigma \to 0^+}{\rightarrow} 0$, and we want to prove that $\xstar \underset{\sigma\to 0^+}{\rightarrow}x^\dagger$. Let $\sigmax > 0$ and $\sigma \in (0, \sigmax]$.
Using the optimality of $\xstar$,
\begin{align}
    F_{\sigma}(\xstar) &\leq F_\sigma(x^\dagger)\\
    f(\xstar) - \tau\log p_\sigma(\xstar) + \tfrac{\lambda(\sigma^2)}{2}\lnorm{\xstar - u}^2 &\leq f(x^\dagger) - \tau\log p_\sigma(x^\dagger) %
    + \tfrac{\lambda(\sigma^2)}{2}\lnorm{x^\dagger - u}^2.
\end{align}
Reordering the terms, we obtain
\begin{equation}
    \lnorm{\xstar - u}^2-\lnorm{x^\dagger - u}^2 \leq \dfrac{2}{\lambda(\sigma^2)}\left[f(x^\dagger) - \tau\log p_\sigma(x^\dagger) - f(\xstar) + \tau\log p_\sigma(\xstar)\right]\label{eq:interm_strg_min_cv}\,.
\end{equation}
Moreover, since $x^\dagger\in \argmin F$,
\begin{equation}
    \label{eq:dagger_min_star}
    F(x^\dagger) = f(x^\dagger) - \tau\log p(x^\dagger) \leq  f(\xstar) - \tau\log p(\xstar) = F(\xstar)\,,
\end{equation}
and plugging \eqref{eq:dagger_min_star} into \eqref{eq:interm_strg_min_cv}, we obtain
\begin{equation*}
\lnorm{\xstar - u}^2-\lnorm{x^\dagger - u}^2 \leq \dfrac{2\tau}{\lambda(\sigma^2)}\left[\log p(x^\dagger)- \log p_\sigma(x^\dagger) + \log p_\sigma(\xstar)-\log p(\xstar)\right]\,.
\end{equation*}
Using \Cref{lemm:eqpsigp}, with $\sigmax$ and $K:=\mathcal{B}(0,R)$ the compact containing $(\xstar)_{\sigma \in (0,\sigmax]}$ as well as $\argmin F$, we have $\tilde{M}>0$ such that, for any $\sigma \in [0,\sigmax]$ :
\begin{align*}
    &|\log p(x^\dagger)- \log p_\sigma(x^\dagger) + \log p_\sigma(\xstar)-\log p(\xstar)| \\
    \leq&  \frac{\sigma^2}{2}\left|\Delta \log p(\xstar)+\lnorm{\nabla \log p(\xstar)}^2-\Delta \log p(x^\dagger)-\lnorm{\nabla \log p(x^\dagger)}^2\right|+2\tilde{M}\sigma^3
\end{align*}
Since $\nabla \log p$ and $\Delta \log p$ are continuous, there exists $B>0$ such that
$\left|\Delta \log p+\lnorm{\nabla \log p}^2-\Delta \log p(x^\dagger)-\lnorm{\nabla \log p(x^\dagger)}^2\right|\leq B$ on $K$.
Thus, 
\begin{align*}
    \lnorm{\xstar - u}^2-\lnorm{x^\dagger - u}^2 &\leq \dfrac{2\tau}{\lambda(\sigma^2)}\left[\log p(x^\dagger)- \log p_\sigma(x^\dagger) + \log p_\sigma(\xstar)-\log p(\xstar)\right]\\
    &\leq  \dfrac{B\tau\sigma^2 + 4 \tau\tilde{M} \sigma^3}{\lambda(\sigma^2)} \underset{\sigma\to 0}{\longrightarrow} 0 \quad\text{because } \dfrac{\sigma^2}{\lambda(\sigma^2)} \underset{\sigma\to 0}{\longrightarrow} 0\,.
\end{align*}
We obtain that any cluster point of $(\xstar)_\sigma$ is at least as close to $u$ as $x^\dagger$. 
A cluster point $\bar{x}$ of $(\xstar)_\sigma$ when $\sigma \to 0$ is thus a minimizer of $F$ ($(i)$) such that $\lnorm{\bar{x}-u}\leq \lnorm{x^\dagger -u}$, therefore, by uniqueness of $x^\dagger$, $\bar{x}=x^\dagger$. Finally, $x^\dagger$ is the only subsequential limit of the bounded sequence $(\xstar)$ when $\sigma \to 0$, so that $\xstar \underset{\sigma \to 0}{\longrightarrow} x^\dagger$.
\end{enumerate}
\end{proof}

\begin{proof}[Proof of \Cref{lemm:eqpsigp}]
We begin by using Taylor-Lagrange's Theorem in order to obtain an expansion of $p_\sigma(x)$ in the variable $\sigma$. Denoting $Z\sim \NN{}$ and using $\E[Z]=0$,
\begin{align*}
& p_\sigma(x) \\
 &  = \E\left[p(x+\sigma Z)\right]\\
 &  =\E\left[p(x) + \sigma\nabla p(x)^T Z +\frac{\sigma^2}{2} Z^T \nabla^2 p(x) Z+ \sigma^3\int_0^1\dfrac{(1 - t)^2}{2}\sum_{i,j,k}\partial_{i,j,k}p(x+t\sigma Z)Z_iZ_jZ_k\dd{t}\right]\\
   & =p(x) + \frac{\sigma^2}{2}\E\left[ Z^T \nabla^2 p(x) Z\right]+\sigma^3\tilde{Q}(\sigma,x)\\
    \end{align*}
where
\begin{equation*}
    \tilde{Q}:(\sigma,x)\in \R_+\times \R^d \mapsto \int_0^1\dfrac{(1 - t)^2}{2}\E\left[\sum_{1\leq i,j,k\leq d} Z_iZ_jZ_k\,\partial_{i,j,k}p(x+t\sigma Z)\right]\dd{t}
\end{equation*}
which we will show to be well defined and a continuous function. Moreover, 
\begin{equation*}
    \E[Z^T \nabla^2 p(x) Z]=\sum_{1\leq i,j \leq d}\E[Z_i Z_j \partial_{ i,j } p(x)] =\sum_i \partial_{i,i} p(x) =\Delta p(x).
\end{equation*}
Notice now that, for any $x\in\R^d$,
\begin{align*}
    \nabla p(x) &= p(x)\nabla \log p(x)\\
   \nabla^2p(x)_{i,j} &= \partial_{i,j} p (x)\\ 
   &=p(x)\left[(\nabla^2 \log p(x))_{i,j} + (\nabla \log p(x))_i(\nabla \log p(x))_j\right] \\
   (\nabla^3 p(x))_{i,j,k}&=\partial_{i,j,k} p(x)\\
   &= p(x)[(\nabla^3\log p(x))_{i,j,k} \\
 \quad  & \quad + (\nabla^2 \log p (x))_{i,j}(\nabla \log p(x))_k+ (\nabla^2 \log p (x))_{j,k} (\nabla \log p(x))_i \\
 \quad  & \quad + (\nabla^2 \log p (x))_{k,i} (\nabla \log p(x))_j + (\nabla \log p(x))_i(\nabla \log p(x))_j(\nabla \log p(x))_k].\\
\end{align*}
As a consequence,
\begin{equation*}
    \Delta p(x) = \sum_{i=1}^d \partial_{i,i} p(x) = p(x) \left(\sum_{i=1}^d \partial_{i,i} \log p(x) + \partial_i \log p(x)^2 \right) = p(x) \left(\Delta \log p(x) + \lnorm{\nabla \log p(x)}^2 \right).
\end{equation*}

We now justify that $\tilde{Q}$ is well-defined and continuous.

First, $\ds(\sigma,z)\mapsto \int_0^1\dfrac{(1 - t)^2}{2}\sum_{1\leq i,j,k\leq d} z_iz_jz_k\,\partial_{i,j,k}p(x+t\sigma z)\dd t $ is clearly continuous. 

Then, since $\nabla^3 \log p$ is uniformly bounded, there exist $M_0, M_1, M_2>0$ such that, for all $x\in \R^d$ :
\begin{itemize}
    \item $\max_{i,j}| \partial_{i,j} \log p(x)|\leq M_2(1+\lnorm{x})$
    \item $\max_i |\partial_i \log p(x)|\leq M_1(1+\lnorm{x}^2)$
    \item $ |\log p(x)| \leq M_0(1+\lnorm{x}^3)$,
\end{itemize}
so that
\begin{equation*}
    |\partial_{i,j,k} p(x)| \leq p(x)\left(M+3M_1M_2(1+\lnorm{x})(1+\lnorm{x}^2)+M_1^3(1+\lnorm{x}^2)^3 \right)=p(x) M_{\text{tot}}(1+\lnorm{x}^6),
\end{equation*}
for a certain $M_{\text{tot}}>0$.

Thus, denoting $m>0$ an upper bound on $p$ (that exists by coercivity of $-\log p$, cf \Cref{lemm:coerc}),
\begin{align*}
     &\E\left[\left|\sum_{1\leq i,j,k\leq d} Z_iZ_jZ_k\,\partial_{i,j,k}p(x+t\sigma Z)\right|\right] \nonumber \\& \hspace{4cm}\leq  \E\left[\sum_{1\leq i,j,k\leq d} \left|Z_iZ_jZ_k\,\partial_{i,j,k}p(x+t\sigma Z)\right|\right] \\
    & \hspace{4cm} \leq M_{\text{tot}} \E\left[p(x+t\sigma Z) \sum_{1\leq i,j,k\leq d} \left|Z_iZ_jZ_k\right|\,\left(1+\lnorm{x+t\sigma Z}^6\right)\right]\\
    & \hspace{4cm}\leq M_{\text{tot}}m \E\left[\sum_{1\leq i,j,k\leq d} \left|Z_iZ_jZ_k\right|\,\left(1+(\lnorm{x}+\lnorm{\sigma Z})^6\right)\right].
\end{align*}
Let $\tilde{\sigma}>0$ and $\tilde{R}>0$. For any $(x,\sigma)\in \mathcal{B}(0,\tilde{R})\times [0,\tilde{\sigma}]$, 
\begin{equation*}
    \E\left[\sum_{1\leq i,j,k\leq d} \left|Z_iZ_jZ_k\right|\,\left(1+(\lnorm{x}+\lnorm{\sigma Z})^6\right)\right]\leq \E\left[\sum_{1\leq i,j,k\leq d} \left|Z_iZ_jZ_k\right|\,\left(1+(\tilde{R}+\lnorm{\tilde\sigma Z})^6\right)\right]
\end{equation*}
which is finite since $Z\sim \NN{}$. By the dominated convergence theorem, $\tilde Q$ is therefore continuous on $\R_+\times \R^d$, and we eventually obtain
\begin{align*}
    p_\sigma(x) &= p(x)\left(1+ \frac{\sigma^2}{2}\left(\Delta \log p(x) + \lnorm{\nabla \log p(x)}^2 \right)+\sigma^3\frac{\tilde{Q}(\sigma,x)}{p(x)}\right).\\
\end{align*}
\,
Let $K\subset \R^d$ compact and $\sigmax>0$. 
Consider the mapping $h:(\sigma,x)\in\R_+\times \R^d\mapsto \frac12\left(\Delta \log p(x) + \lnorm{\nabla \log p(x)}^2 \right)+\sigma\frac{\tilde{Q}(\sigma,x)}{p(x)}$, which inherits continuity from $\Delta \log p$, $\nabla \log p$, $\tilde{Q}$ and from $p > 0$, and is therefore bounded on $[0,\sigmax]\times K$. Then,
\begin{align*}
&\left|
\log\left(\frac{p_\sigma(x)}{p(x)}\right)
-\frac{\sigma^2}{2}
\left(\Delta\log p(x)+\lnorm{\nabla\log p(x)}^2\right)
\right|\\
&\hspace{4cm}
=\left|
\log\left(
1+\frac{\sigma^2}{2}
\left(\Delta\log p(x)+\lnorm{\nabla\log p(x)}^2\right)
+\sigma^3\frac{\tilde Q(\sigma,x)}{p(x)}
\right)\right.\\
&\hspace{5cm}\left.
-\frac{\sigma^2}{2}
\left(\Delta\log p(x)+\lnorm{\nabla\log p(x)}^2\right)
\right|\\
&\hspace{4cm}\leq 
\underbrace{
\left|\log(1+\sigma^2 h(\sigma,x))-\sigma^2 h(\sigma,x)\right|
}_{\text{bounded by a term $\propto\sigma^4$ since $h$ remains bounded}}
+\sigma^3
\underbrace{
\left|\frac{\tilde Q(\sigma,x)}{p(x)}\right|,
}_{\substack{\text{continuous, thus bounded on}\\
[0,\sigmax]\times K}}
\end{align*}
and we obtain the result.

\end{proof}

\subsection{Bound on the derivative of the minimizer path}
\reglips*
\begin{proof}
We show in fact that for any $\sigmax>0$ and $0<\sigma_1\leq\sigma_2\leq \sigmax$,
\begin{equation}
    \lnorm{x_{\sigma_2}^* - x_{\sigma_1}^*} \leq \left(T(\sigmax)+ \dfrac{M\tau\sqrt{d}+L_f T(\sigmax)}{2\lambda(\sigma_1^2)}\right)(\sigma_2^2 - \sigma_1^2) + r(\sigmax)\log\left(\dfrac{\lambda(\sigma_2^2)}{\lambda(\sigma_1^2)} \right), 
\end{equation}
where $r(\sigmax) = \underset{\sigma\in(0,\sigmax]}{\sup}\lnorm{\xstar - u}$ and $T(\sigmax) = \underset{\sigma\in(0,\sigmax]}{\sup} \frac{1}{\tau}\lnorm{\nabla f(\xstar)+\lambda(\sigma^2)(\xstar -u)}$ both quantity being finite.
By convexity and differentiability of $F_{\sigma}$, $\xstar$ is characterized by $\nabla F_\sigma(\xstar) = 0$. 
By strong convexity of $F_\sigma$, $\nabla^2 F_\sigma(x_\sigma^*)$ is invertible, hence the implicit function theorem \citep[Theorem 3.3.1]{krantz2002implicit}
states that the mapping $\sigma^2\mapsto \xstar$ is continuously differentiable, and that its derivative, denoted $\dottedxstar$, satisfies
\begin{equation}\label{eq:dottedxstar}
    \dottedxstar = -[\nabla^2F_\sigma(\xstar)]^{-1}\left(\partial_{\sigma^2}\nabla F_\sigma\right)(\xstar).
\end{equation}
\begin{align*}
    \nabla^2F_\sigma(\xstar) &= \nabla^2 f(\xstar) - \tau\nabla^2\log p_\sigma(\xstar) + \lambda(\sigma^2)\Id\\
    \partial_{\sigma^2}\nabla F_\sigma(\xstar) &= -\tau\partial_{\sigma^2}\nabla\log p_{\sigma}(\xstar) + \lambda'(\sigma^2) (\xstar - u)\label{eq:hessian_F}\\
    &\stackclap{\circled{a}}{=} -\dfrac{\tau}{2}\nabla\Delta\log p_{\sigma}(\xstar) - \tau[\nabla^2\log p_\sigma(\xstar)]\nabla\log p_\sigma(\xstar) + \lambda'(\sigma^2) (\xstar - u).%
\end{align*}
where \circled{a} holds by \citet[Lemma 3]{pesme2025map}. 
For ease of notation, denote $M_\sigma := \nabla^2f(\xstar) + \lambda(\sigma^2)\Id$ and $N_\sigma := -\tau\nabla^2\log p_\sigma(\xstar)$ for $\sigma>0$, which are two symmetric matrices, respectively positive definite and positive semi-definite.

We first want to bound $\|\dottedxstar\|$ on all $[0, \sigmax]$. 
We begin by expanding \Cref{eq:dottedxstar} explicitly,
\begin{equation}
    \dottedxstar = -[M_\sigma + N_\sigma]^{-1}
    \left(-\dfrac{\tau}{2}\nabla\Delta\log p_{\sigma}(\xstar) + N_\sigma\nabla\log p_\sigma(\xstar) + \lambda'(\sigma^2) (\xstar - u) \right).
\end{equation}
Let us recall a result proved in \citet[Lemma 5]{pesme2025map}: under \Cref{ass:prior}, $\lnorm{\nabla\Delta\log p_{\sigma}}$ is uniformly bounded by $ M \sqrt{d}$.
Moreover, we can bound the operator norm $\lnorm{\cdot}_{\mathrm{op}}$ of $[M_\sigma + N_\sigma]^{-1}$ and $[M_\sigma + N_\sigma]^{-1}N_\sigma$:
\begin{itemize}
    \item  Using Loewner order, $M_\sigma + N_\sigma \succeq M_\sigma \succeq \lambda(\sigma^2) \Id$, which implies that $[M_\sigma + N_\sigma]^{-1}\preceq \frac{1}{\lambda(\sigma^2)}\Id$, so that $\lnorm{[M_\sigma + N_\sigma]^{-1}}_{\mathrm{op}}\leq  \frac{1}{\lambda(\sigma^2)}$.
    \item For any $x\in\R^d$, $\sqrt{\lambda(\sigma^2)}\lnorm{x}\leq \lnorm{M_\sigma^{1/2}x}\leq \sqrt{L_f+\lambda(\sigma^2)}\lnorm{x}$, so that 
    \begin{equation*}
        \frac{1}{\sqrt{L_f+\lambda(\sigma^2)}}\lnorm{M_\sigma^{1/2}x}\leq \lnorm{x} \leq \frac{1}{\sqrt{\lambda(\sigma^2)}}\lnorm{M_\sigma^{1/2} x}.
    \end{equation*} Let $H:=[M_\sigma + N_\sigma]^{-1}N_\sigma$, then
    \begin{align*}
        \lnorm{H}_{\mathrm{op}}&= \underset{x\neq 0}{\sup}\quad \dfrac{\lnorm{Hx}}{\lnorm{x}}\\ 
        &\leq\underset{x\neq 0}{\sup}\quad \dfrac{(\lambda(\sigma^2))^{-1/2}\lnorm{M_\sigma^{1/2}Hx}}{(L_f+\lambda(\sigma^2))^{-1/2}\lnorm{M_\sigma^{1/2}x}}\\
        &=\sqrt{1+\dfrac{L_f}{\lambda(\sigma^2)}}\quad \underset{z\neq 0}{\sup}\quad \dfrac{\lnorm{M_\sigma^{1/2}HM_\sigma^{-1/2}z}}{\lnorm{z}}\\
        &=\sqrt{1+\dfrac{L_f}{\lambda(\sigma^2)}}\quad \lnorm{M_\sigma^{1/2}HM_\sigma^{-1/2}}_{\mathrm{op}}.
    \end{align*}
    Notice now that $M_\sigma^{1/2}HM_\sigma^{-1/2}=[\Id + M_\sigma^{-1/2}N_\sigma M_\sigma^{-1/2}]^{-1}M_\sigma^{-1/2}N_\sigma M_\sigma^{-1/2}$ is a symmetric matrix whose eigenvalues are bounded by $1$, so that its operator norm is at most $1$. 
    Hence, $\lnorm{H}_{\mathrm{op}}\leq  \sqrt{1+\frac{L_f}{\lambda(\sigma^2)}}\leq  1+\frac{L_f}{2\lambda(\sigma^2)}$.
\end{itemize}
We also know that, by optimality of $\xstar$, 
\begin{equation}
    \nabla \log p_\sigma(\xstar) = \frac{1}{\tau}\left(\nabla f(\xstar)+\lambda(\sigma^2)(\xstar -u) \right),
\end{equation}
Therefore, by combining the previous facts, the quantity $\lnorm{\dottedxstar}$ is equal to
\begin{align}
   &\lnorm{-[M_\sigma + N_\sigma]^{-1}
    \left(-\dfrac{\tau}{2}\nabla\Delta\log p_{\sigma}(\xstar) + \frac{1}{\tau}N_\sigma\Big(\nabla f(\xstar)+\lambda(\sigma^2)(\xstar -u) \Big) + \lambda'(\sigma^2) (\xstar - u) \right)} \nonumber\\
    &\leq \frac{1+\frac{L_f}{2\lambda(\sigma^2)}}{\tau}\|\nabla f(\xstar)+\lambda(\sigma^2)(\xstar -u)\| + \dfrac{1}{\lambda(\sigma^2)}\left(\dfrac{M\tau\sqrt{d}}{2}+\lambda'(\sigma^2) \|\xstar - u\|\right)
\end{align}
Moreover, following \Cref{prop:lim_minim}, we know that $\xstar$ is bounded on $(0,\sigmax]$, and since $\lambda$ is continuous in $0$, we can define the finite quantities $T(\sigmax):=\underset{\sigma\in(0,\sigmax]}{\sup}\frac{1}{\tau}\|\nabla f(\xstar)+\lambda(\sigma^2)(\xstar -u)\|$ and $r(\sigmax) := \underset{\sigma\in(0,\sigmax]}{\sup} \lnorm{\xstar - u}$. Then,
\begin{equation*}
    \|\dottedxstar\| \leq T(\sigmax) + \dfrac{1}{\lambda(\sigma^2)}\left(\dfrac{M\tau\sqrt{d}}{2}+\lambda'(\sigma^2) r(\sigmax)+ \dfrac{L_f T(\sigmax)}{2}\right), 
\end{equation*}
Recall that $\dottedxstar := \dfrac{\dd{\xstar}}{\dd{\sigma^2}}$. 
Then, by triangular inequality, for $0<\sigma_1\leq\sigma_2\leq \sigmax$:
\begin{align*}
    \| x_{\sigma_2}^* - x_{\sigma_1}^* \| &= \lnorm{\int_{\sigma_1^2}^{\sigma_2^2} \dottedxstar \dd{\sigma^2} } \\
    &\leq \int_{\sigma_1^2}^{\sigma_2^2} \|\dottedxstar\|\dd{\sigma^2} \\
    &\leq \int_{\sigma_1^2}^{\sigma_2^2}  T(\sigmax) + \dfrac{1}{\lambda(\sigma^2)}\left(\dfrac{M\tau\sqrt{d}}{2}+\lambda'(\sigma^2) r(\sigmax)+ \dfrac{L_f T(\sigmax)}{2}\right)\dd{\sigma^2} \\
    &\leq \left(T(\sigmax)+ \dfrac{M\tau\sqrt{d}+L_fT(\sigmax)}{2\lambda(\sigma_1^2)}\right)(\sigma_2^2 - \sigma_1^2) + r(\sigmax)\log\left(\dfrac{\lambda(\sigma_2^2)}{\lambda(\sigma_1^2)} \right).  \\
\end{align*}

\end{proof}

\section{Proofs related to the convergence of the algorithm \texorpdfstring{(\cref{ssec:cv})}{}}
\label[appendix]{sec:proofcvalgo}

\subsection{A useful lemma}
Before presenting the proof of the main theorem, we recall a lemma which will be essential to the proof of convergence in \Cref{thm:iterates_cv,thm:renoised_iterates_cv}.

\begin{lemma}\label{lem:suitecv}
    Let $(\beta_k)$ be a sequence of non-negative real numbers satisfying
    \begin{equation}
        \beta_{k + 1} \leq (1 - \gamma_k)\beta_k + \gamma_k \delta_k
    \end{equation}
    where $(\gamma_k), (\delta_k)$ satisfy the conditions:
    \begin{enumerate}[label=(\roman*)]
        \item $(\gamma_k)\subset [0,1]$ and $\sum_{k}\gamma_k = +\infty$
        \item $\delta_k \underset{k\to+\infty}{\longrightarrow} 0$
    \end{enumerate}
    Then $\lim_k \beta_k = 0$.
\end{lemma}
This is a specific case of the more general result of \citet[Lemma 2.5]{xu2002iterative}.

\subsection{Convergence of the deterministic iterates}\label{app:pf_gamma}

\iteratescv*

\begin{proof}
The first goal of this proof is to apply \Cref{lem:suitecv} 
to the sequence $\beta_{k+1}:=\lnorm{x_{k + 1} - \xstar[k]}$.

We first recall that the step $k$ of \eqref{eq:genmmse} is a gradient descent step on the objective $F_{\sigma_k}$ at $x_k$ with step size $\alpha_k$. Moreover, by \Cref{prop:conditionning}, $F_{\sigma_k}$ is $\mu_{\sigma_k}-$strongly convex, $L_{\sigma_k}-$smooth and twice differentiable, with $\mu_{\sigma_k}=\lambda_k$ and $L_{\sigma_k}=L_f+\frac{\tau}{\sigma_k^2}+\lambda_k$.
Hence, since $0<\alpha_k<1<\frac{1}{\mu_{\sigma_k}}=\frac{1}{\lambda_k}$, a classical result on gradient descent \citep[Section 1.2.3]{nesterov2013introductory} yields
\begin{equation}\label{eq:x_kgraddescent}
    \lnorm{x_{k + 1} - \xstar[k]} \leq \max(1 - \alpha_k\lambda_k, \alpha_k L_{\sigma_k} - 1)\lnorm{x_k - \xstar[k]}
\end{equation}
Since with our choice of $\alpha_k$, $\alpha_k L_{\sigma_k}-1\underset{k\to + \infty}{\longrightarrow}0$, and $1-\alpha_k \lambda_k\underset{k\to + \infty}{\longrightarrow} 1$, the inequality $\alpha_k L_k -1 \leq 1-\alpha_k \lambda_k$ holds for all sufficiently large $k$. 
Up to starting the analysis at a later iteration, we may assume that it holds for all $k\geq 0$. 
Thus, by triangular inequality,
\begin{equation}\label{eq:inegtrig}
    \lnorm{x_{k + 1} - \xstar[k]} \leq (1 - \alpha_k\lambda_k)\lnorm{x_k - \xstar[k - 1]} + (1 - \alpha_k\lambda_k)\lnorm{\xstar[k - 1] - \xstar[k]},
\end{equation}
and, using \Cref{prop:reglips}, we can bound $\lnorm{\xstar[k - 1] - \xstar[k]}$ by
\begin{equation}\label{eq:borne_pseudolip}
    \lnorm{\xstar[k - 1] - \xstar[k]} \leq \left(C_1+ \dfrac{C_3}{\lambda_k}\right)(\sigma_{k-1}^2 - \sigma_k^2) + C_2\log\left(\dfrac{\lambda_{k-1}}{\lambda_k} \right),
\end{equation}
where $C_1,C_2,C_3>0$.
To ease the notation, we denote:
\begin{equation}
    B_{k} = C_1+ \dfrac{C_3}{\lambda_{k}}\quad \text{ and }\quad D_{k}=C_2\log\left(\dfrac{\lambda_{k-1}}{\lambda_k} \right).
\end{equation}
Combining  \eqref{eq:inegtrig} and \eqref{eq:borne_pseudolip}, and denoting $\beta_{k + 1} := \lnorm{x_{k + 1} - \xstar[k]}$, we obtain the following recursive relation
\begin{equation}
    \beta_{k + 1} \leq (1 - \alpha_k\lambda_k)\beta_k + \underbrace{(1 - \alpha_k\lambda_k)\left[B_k(\sigma_{k - 1}^2 - \sigma_k^2) + D_k\right]}_{=:\omega_k}
\end{equation}
We recognize the recursive relation considered in \Cref{lem:suitecv} with $\gamma_k = \alpha_k\lambda_k$ and $\delta_k = \frac{\omega_k}{\alpha_k\lambda_k}$. By \Cref{lem:suitecv}, we have $\beta_k \longrightarrow 0$ under the following conditions:
\begin{equation}
    \sum_{k = 1}^{+\infty}\alpha_k\lambda_k = +\infty, \quad \dfrac{\omega_k}{\alpha_k\lambda_k} \underset{k\to+\infty}{\longrightarrow} 0
\end{equation}
Using, $\alpha_k \sim \frac{\sigma_k^2}{\tau}$ and basic equivalences, 
the previous conditions are equivalent to the constraints \emph{(i)}, \emph{(ii)} and \emph{(iii)}. 
Under these constraints, we thus have $\beta_{k + 1} = \lnorm{x_{k + 1} - \xstar[k]}\underset{k\to +\infty}{\longrightarrow}0$. By \Cref{prop:lim_minim}, the sequence $(\xstar[k])_k$ is bounded. Since
\begin{equation}
\lnorm{x_k} \leq \lnorm{x_k-\xstar[k-1]}+\lnorm{\xstar[k-1]} =\beta_k+\lnorm{\xstar[k-1]},
\end{equation}
the sequence $(x_k)_k$ is also bounded.
Moreover by triangular inequality, for any $x^*\in \R^d$ :
\begin{equation}\label{eq:triang_cluster}
    \lnorm{\xstar[k-1] - x^*}-\lnorm{x_k - \xstar[k-1]}\leq \lnorm{x_k - x^*} \leq \lnorm{x_k - \xstar[k-1]} + \lnorm{\xstar[k-1] - x^*}.
\end{equation}
By letting $x^*$ be a cluster point of $(x_k)_k$ (resp. $(x_{\sigma_k})_k$ and passing to the limit in the left (resp. right) hand side of \eqref{eq:triang_cluster} with a corresponding convergent subsequence, we get that the cluster points of $(x_k)_k$ are exactly the cluster points of $(\xstar[k])_k$. 
By \Cref{prop:lim_minim}, they are minimizers of $F$, and if in addition we suppose $\frac{\sigma_k^2}{\lambda_k}\underset{k\to +\infty}{\longrightarrow}0$, then $(\xstar[k])_k$ converges to $x^* = x^\dagger$ and so does $(x_k)_k$.

\end{proof}

\subsection{Convergence of Renoised-\genmmse{}}\label{app:pf_renoised}
\renoisediteratescv*

\begin{proof}
For ease of notation, let $\varepsilon \sim \NN{}$. For $\sigma>0$, the function $G_{\sigma}: x \in \R^d \mapsto f_\sigma(x) - \tau \E[\log p_\sigma(x + \sigma\varepsilon)]$ is a smoothing of $F_{\sigma}$. 

As the proof of \citet[Lemma 5]{pesme2025map} shows that $\lnorm{\nabla^3 \log p_{\sigma} }_F \leq \lnorm{\nabla^3 \log p}_F\leq M$, by the dominated convergence theorem, the mapping $x\mapsto -\tau \E[\log p_\sigma(x + \sigma\varepsilon)]$ is well defined and thrice continuously differentiable, and its successive derivatives write
\begin{equation}
    -\tau\E[\nabla \log p_\sigma(\cdot + \sigma \varepsilon)], \quad     -\tau\E[\nabla^2 \log p_\sigma(\cdot + \sigma \varepsilon)],  \quad    -\tau \E[\nabla ^3\log p_\sigma(\cdot + \sigma \varepsilon)],
\end{equation}
so that it inherits convexity and $\frac{\tau}{\sigma^2}-$smoothness from $-\tau\log p_\sigma$.
Thus, $G_\sigma$ is $L_{\sigma}-$smooth and $\lambda(\sigma^2)-$strongly convex like $F_\sigma$, with $L_\sigma = \lambda(\sigma^2) + \frac{\tau}{\sigma^2}+L_f$. 

Our goal is to prove an inequality similar to \eqref{eq:x_kgraddescent}, in order to show that $\lnorm{\tx_{k+1}-\xstar[k]}\longrightarrow 0$. As $\xstar$ is a minimizer of $F_\sigma$ and the iterates are obtained by gradient descent steps on the $G_{\sigma_k}$, the descent lemma does not apply directly.

We first recall some basic inequalities that will be useful in the proof.
\begin{itemize}
    \item Young's inequality: for $a, b\in\R$ and $\delta > 0$,
        \begin{equation}
            |ab|\leq \dfrac{\delta a^2}{2} + \dfrac{b^2}{2\delta}\,
            \tag{A}\label{eq:ineqa}
        \end{equation}
    \item As a consequence, using Cauchy-Schwarz, for any $x, y, z\in \R^d$ and $\delta > 0$
        \begin{equation}
            \lnorm{x - y}^2 \leq (1 + \delta)\lnorm{x - z}^2 + \left(1 + \frac{1}{\delta}\right)\lnorm{y - z}^2\,.\tag{B}\label{eq:ineqb}
        \end{equation}
    \item For $h$ an $L$-smooth and $\mu$-strongly convex function, and $x,y\in\R^d$ \citep[Thm. 2.1.12]{nesterov2013introductory}
        \begin{equation}
            \langle\nabla h(x) - \nabla h(y), x - y\rangle \geq \dfrac{\mu L}{\mu + L}\lnorm{x - y}^2 + \dfrac{1}{\mu + L}\lnorm{\nabla h(x) - \nabla h(y)}^2\,.\tag{C}\label{eq:ineqc}
        \end{equation}
\end{itemize}
By definition of our iterates, 
\begin{align}
    \lnorm{\tx_{k + 1} - \xstar[k]}^2&=\lnorm{\tx_k - \alpha_k\nabla G_{\sigma_k}(\tx_k) - \xstar[k]}^2 \notag \\
    &= \lnorm{\tx_k - \xstar[k]}^2 - \underbrace{2\alpha_k\langle\nabla G_{\sigma_k}(\tx_k), \tx_k - \xstar[k]\rangle}_{\circled{1}} + \alpha_k^2\underbrace{\lnorm{\nabla G_{\sigma_k}(\tx_k)}^2}_{\circled{2}}\,.\label{eq:strgineq}
\end{align}

We first bound \circled{1}. 
By Cauchy-Schwarz inequality, 
\begin{equation}
    \langle\nabla G_{\sigma_k}(\tx_k), \tx_k - \xstar[k]\rangle \geq \langle\nabla G_{\sigma_k}(\tx_k)-\nabla G_{\sigma_k}(\xstar[k]), \tx_k - \xstar[k]\rangle - \lnorm{G_{\sigma_k}(\xstar[k]}\lnorm{\tx_k - \xstar[k]}.
\end{equation}
Let $r_k := \lnorm{\tx_k - \xstar[k]}$ and $\Delta_k:= \lnorm{\nabla G_{\sigma_k}(\tx_k) - \nabla G_{\sigma_k}(\xstar[k])}$. Then, using \eqref{eq:ineqc} on $G_{\sigma_k}$:
\begin{equation}
    \langle\nabla G_{\sigma_k}(\tx_k) - \nabla G_{\sigma_k}(\xstar[k]),\tx_k - \xstar[k] \rangle \geq \dfrac{\lambda(\sigma_k^2) L_{\sigma_k}}{\lambda(\sigma_k^2) + L_{\sigma_k}}r_k^2 + \dfrac{1}{\lambda(\sigma_k^2) + L_{\sigma_k}}\Delta_k^2.
\end{equation}
For notational convenience, from now on $\lambda_k := \lambda(\sigma_k^2)$ and $L_k:= L_{\sigma_k}$.
Combining the two preceding inequalities and multiplying both sides by $2\alpha_k$, we obtain
\begin{align}
     2\alpha_k\langle\nabla G_{\sigma_k}(\tx_k),\tx_k - \xstar[k]\rangle &\geq \dfrac{2\alpha_k\lambda_kL_k}{\lambda_k + L_k}r_k^2 + \dfrac{2\alpha_k}{\lambda_k + L_k}\Delta_k^2 - 2\alpha_k\lnorm{\nabla G_{\sigma_k}(\xstar[k])}r_k \\
     &\stackclap{\circled{a}}{\geq} \left(\dfrac{2\alpha_k\lambda_kL_k}{\lambda_k + L_k} - \dfrac{\delta}{2}\right)r_k^2 + \dfrac{2\alpha_k}{\lambda_k + L_k}\Delta_k^2 - \dfrac{2\alpha_k^2}{\delta}\lnorm{\nabla G_{\sigma_k}(\xstar[k])}^2 \label{eq:renoised_scalar_term}
\end{align}
where \circled{a} holds for any $\delta > 0$ by \eqref{eq:ineqa} 
with $a =  r_k= $ and $b = \alpha_k \Vert \nabla G_{\sigma_k}(\xstar[k]) \Vert$.

Now, we bound \circled{2} using \eqref{eq:ineqb} with $\delta = \frac{1}{2}$:
\begin{equation}
    \label{eq:perfect_oracle_gradient_bound}
    \lnorm{\nabla G_{\sigma_k}(\tx_k)}^2 = \lnorm{\nabla G_{\sigma_k}(\tx_k) - \nabla G_{\sigma_k}(\xstar[k]) + \nabla G_{\sigma_k}(\xstar[k])}^2 \leq \dfrac32\Delta_k^2 + 3\lnorm{\nabla G_{\sigma_k}(\xstar[k])}^2\,.
\end{equation}

Finally, we can gather the two bounds \circled{1}-\circled{2}:
\begin{align}
    \lnorm{\tx_{k + 1} - \xstar[k]}^2 &\leq \left(1 - 2\dfrac{\alpha_k\lambda_k L_k}{\lambda_k + L_k} + \dfrac{\delta}{2}\right)r_k^2\notag\\
    & + \underbrace{\alpha_k\left(\dfrac32\alpha_k - \dfrac{2}{\lambda_k + L_k}\right)}_{\leq 0}\Delta_k^2 \notag\\ & + \left(3\alpha_k^2 + \dfrac{2\alpha_k^2}{\delta}\right)\lnorm{\nabla G_{\sigma_k}(\xstar[k])}^2
\end{align}
As $\alpha_k(\lambda_k + L_k) \underset{k\to +\infty}{\longrightarrow} 1$, up to starting at a later iteration, the coefficient of $\Delta_k^2$ is negative and can be dropped. Let $h_k = \frac{\alpha_k\lambda_k L_k}{\lambda_k + L_k}$, and choose $\delta = 2h_k$, so that
\begin{equation}
    \label{eq:stochbound}
\lnorm{\tx_{k + 1} - \xstar[k]}^2 \leq \left(1 - h_k\right)r_k^2 + \left(3\alpha_k^2 + \dfrac{\alpha_k^2}{h_k}\right)\lnorm{\nabla G_{\sigma_k}(\xstar[k])}^2.
\end{equation}
In order to conclude, we need to bound $\lnorm{\nabla G_{\sigma_k}(\xstar[k])}$. By Taylor-Lagrange's theorem at order 2, for any $\epsilon = (\epsilon^i)_{1\leq i \leq d}  \in \R^d$,
\begin{equation}
    \nabla \log p_{\sigma_k}(\xstar[k] + \sigma_k \epsilon) = \nabla \log p_{\sigma_k}(\xstar[k]) + \sigma_k\nabla^2 \log p_{\sigma_k}(\xstar[k])\epsilon + R_k(\epsilon)\,,
\end{equation}
where, for $1\leq i \leq d$,
\begin{equation}
    (R_k(\epsilon))_i=\sigma_k^2  \sum_{1\leq j, l \leq d} \int_0^1(1-t)\partial_{i,j,k}\log p_{\sigma_k}(\xstar[k]+t\sigma_k \epsilon)\epsilon^l\epsilon^j\,dt.
\end{equation}
As already mentioned, the proof of \citet[Lemma 5]{pesme2025map} establishes that $\lnorm{\nabla^3 \log p_{\sigma_k} }_F \leq \lnorm{\nabla^3 \log p}_F\leq M$. Thus, by the Cauchy-Schwarz inequality:
\begin{equation}
    |(R_k(\epsilon))_i| \leq \dfrac{\sigma_k^2}{2}  \sum_{1\leq j, l \leq d} |\partial_{i,j,k}\log p_{\sigma_k}(\xstar[k]+t\sigma_k \epsilon)|\lnorm{\epsilon}_\infty^2 \leq  \dfrac{d\sigma_k^2}{2}\sqrt{M}\lnorm{\epsilon}^2_\infty.
\end{equation}
so that $\lnorm{R_k(\epsilon)}^2\leq \frac{d^3M\sigma_k^4}{4}\lnorm{\epsilon}^4_\infty$.
Since $\nabla F_{\sigma_k}(\xstar[k])=\nabla f_{\sigma_k}(\xstar[k])- \tau\nabla \log p_{\sigma_k}(\xstar[k])=0$ by the optimality of $\xstar[k]$, and $\E_\varepsilon[\nabla^2\log p_{\sigma_k}(\xstar[k])\varepsilon]=0$ by the linearity of expectation, we obtain:
\begin{align}
    \label{eq:perfect_oracle_gradient_critical_point}
    \lnorm{\nabla G_{\sigma_k}(\xstar[k])}^2 &= \lnorm{\nabla f_{\sigma_k}(\xstar[k])- \tau\E_{\varepsilon}[\nabla \log p_{\sigma_k}(\xstar[k] + \sigma_k \varepsilon)]}^2 \notag\\
    &=\tau^2 \lnorm{\E_{\varepsilon}[R_k(\varepsilon)]}^2 \notag \\
    &\leq \tau^2 \E_{\varepsilon}[\lnorm{R_k(\varepsilon)}^2]\notag \\
    &\leq \dfrac{d^3M\sigma_k^4 \tau^2}{4}\E_{\varepsilon}[\lnorm{\varepsilon}^4_\infty],
\end{align}
where the first inequality follows from Jensen's. Let $B:=\dfrac{d^3M\tau^2}{4}\E_{\varepsilon}[\lnorm{\varepsilon}^4_\infty]$. Plugging this into \eqref{eq:stochbound} we obtain,
\begin{equation}
    \lnorm{\tx_{k + 1} - \xstar[k]}^2 \leq \left(1 - h_k\right)r_k^2 + \left(3\alpha_k^2 + \dfrac{\alpha_k^2}{h_k}\right)B\sigma_k^4\,.
\end{equation}
Since $ h_k=\frac{\alpha_k\lambda_kL_k}{\lambda_k+L_k}\geq \frac{\alpha_k\lambda_k}{2}$ and $1>\alpha_k\lambda_k$, we have $(3+\frac{1}{h_k})\leq \frac{5}{\alpha_k \lambda_k}$, and
\begin{equation}
    \label{eq:simplified_renoised_bound}
    \lnorm{\tx_{k + 1} - \xstar[k]}^2 \leq \left(1 - h_k\right)r_k^2 +\dfrac{5\alpha_k B\sigma_k^4}{\lambda_k}\,.
\end{equation}
In order to derive a recursive relation on $\beta_{k + 1} := \lnorm{\tx_{k+1} - \xstar[k]}^2$, we need to make $\beta_k=\lnorm{\tx_k - \xstar[k-1]}^2$ appear on the right. To do so, we apply \eqref{eq:ineqb} to $r_k^2$ with $\delta = \dfrac{h_k}{2 - 2h_k}$, which yields
\begin{align}
\beta_{k+1} &\leq \left(1 - \dfrac{h_k}{2}\right)\beta_k + \left(1 - h_k\right)\left(\dfrac{2}{h_k} - 1\right)\lnorm{\xstar[k]-\xstar[k-1]}^2 +\dfrac{5\alpha_k B\sigma_k^4}{\lambda_k}\\
&\leq \left(1 - \dfrac{h_k}{2}\right)\beta_k + \underbrace{\dfrac{2}{h_k}\lnorm{\xstar[k]-\xstar[k-1]}^2 +\dfrac{5\alpha_k B\sigma_k^4}{\lambda_k}}_{:=h_k \delta_k /2}.
\end{align}
Applying \Cref{lem:suitecv}, $\beta_k \underset{k\to+\infty}{\longrightarrow} 0$ under the following constraints:
\begin{equation}
    \sum_{k = 0}^{+\infty}h_k = +\infty,\quad \delta_k \underset{k\to+\infty}{\longrightarrow} 0.
\end{equation}
Since $h_k\sim \dfrac{\sigma_k^2}{\tau}\lambda_k$, we recover condition \emph{(i)} of \Cref{thm:iterates_cv} from $\sum_{k = 0}^{+\infty}h_k = +\infty$. By definition of $\delta_k$,
\begin{equation}
    \delta_k = \dfrac{4}{h_k^2}\lnorm{\xstar[k]-\xstar[k-1]}^2 + \dfrac{10\alpha_k B\sigma_k^4}{h_k\lambda_k}.
\end{equation}

By \Cref{prop:reglips}, there exist $C_1,C_2, C_3>0$ such that 
\begin{align}
    \dfrac{2}{h_k}\lnorm{\xstar[k]-\xstar[k-1]}&\leq \left(C_1 + \dfrac{C_3}{\lambda_k}\right)\dfrac{\sigma_{k-1}^2 - \sigma_k^2}{2h_k} + \frac{C_2}{2h_k}\log\left(\dfrac{\lambda_{k-1}}{\lambda_k}\right)\\
    &\underset{k\to + \infty}{\sim}\underbrace{\dfrac{C_3\tau(\sigma_{k-1}^2-\sigma_k^2)}{2\sigma_k^2 \lambda_k^2}}_{\longrightarrow 0,\quad \emph{(ii)}}+\underbrace{\frac{C_2\tau}{2\sigma_k^2\lambda_k}\log\left(\dfrac{\lambda_{k-1}}{\lambda_k}\right)}_{\longrightarrow 0,\quad \emph{(iii)}}.
\end{align}
so that the first term converges to $0$ under conditions \emph{(ii)} and \emph{(iii)}.
Moreover, provided that $\frac{\sigma_k^2}{\lambda_k}\to 0$,
\begin{equation}
    \dfrac{10\alpha_k B\sigma_k^4}{h_k\lambda_k} \underset{k\to+\infty}{\sim}\dfrac{10 B \sigma_k^4 }{\lambda_k^2}
\end{equation}
converges to $0$. By \Cref{lem:suitecv}, $\beta_{k+1} =\lnorm{x_{k + 1} - \xstar[k]}^2$ converges to $0$. Using \Cref{prop:lim_minim} concludes the proof.

\end{proof}

\subsection{Stochastic-\genmmse{} and its convergence}
\label{ssec:stochastic_mmse}
In this section, we investigate a stochastic version of \genmmse{}.
In this setting, we do not assume that we have access to $\E_{\varepsilon\sim\NN{}}\left[\mmse (x + \sigma\varepsilon)\right]$; instead, at each iteration, we use a single noise sample to estimate the expectation.

\begin{tcolorbox}[colback=WarmGray!5,colframe=SlateGray,boxrule=0.5pt,arc=4pt]
\textbf{Stochastic-\genmmse{}:} 
    Given annealing noise level sequences $(\sigma_k)_k, (\nu_k)_k$, an annealing step size sequence $(\alpha_k)_k$ and an arbitrary initialization $x_0\in\R^d$, we define the following stochastic iterates  
    \begin{equation}
        \label{eq:stoch_genmmse}
        \tx_{k + 1} = \alpha_k(\tx_k - \nabla f_{\sigma_k}(\tx_k)) + (1 - \alpha_k)\mmse[k](\tx_k + \nu_k\varepsilon_k),
    \end{equation}
    where the $\varepsilon_k \sim \NN{}$ are independent.
\end{tcolorbox}

In this scheme, the noise level $\nu_k$ added to the input of the denoiser is decoupled from the denoiser's noise level $\sigma_k$. 
We prove convergence of such scheme under the same conditions as above, together with an additional condition on $\nu_k$.

\begin{restatable}[Convergence of Stochastic-\genmmse{}]{theorem}{stochasticiteratescv}\label{thm:stochastic_iterates_cv}
      Let $(\sigma_k)_k, (\nu_k)_k\in \R_+^\N$ be sequences strictly decreasing to $0$ and, for any $k\geq 0$, let $\lambda_k:=\lambda(\sigma_k^2)\in (0,1)$ and $\alpha_k := \frac{\sigma_k^2}{\sigma_k^2 + \tau}$. Under constraints (i), (ii), (iii) of \Cref{thm:iterates_cv}, $\frac{\sigma_k^2}{\lambda_k}\to 0$ and $\nu_k^2=o(\sigma_k^2\lambda_k)$, the stochastic iterates $(\tx_k)$ from \eqref{eq:stoch_genmmse} converges in mean-square to $x^\dagger$, i.e,
    \begin{equation}
        \E\lnorm{\tilde x_k - x^\dagger}^2\underset{k\to+\infty}{\longrightarrow} 0\,.
    \end{equation}
\end{restatable}

\begin{proof}

For $k\geq 0$, we define $\mathcal{F}_k$ as the $\sigma-$algebra generated by the random variables $\varepsilon_0, \cdots , \varepsilon_{k-1}$ and denote by $\E_k$ the conditional expectation over $\mathcal{F}_k$. 
By definition of our iterates, for any $k\geq 0$
\begin{equation}
    \tx_{k + 1} = \tx_k - \alpha_k\left( \nabla f_{\sigma_k}(\tx_k)-\tau\nabla \log p_{\sigma_k}(\tx_k+\nu_k \varepsilon_k) \right) + (1 - \alpha_k)\nu_k\varepsilon_k,
\end{equation}
so that $\tx_k$ is $\mathcal{F}_k-$measurable.
We define the following functions: 
\begin{align}
    F_k:=F_{\sigma_k}:x\in\R^d &\mapsto f_{\sigma_k}(x) - \tau\log p_{\sigma_k}(x)\,, \\
    G_k:x\in\R^d &\mapsto f_{\sigma_k}(x) - \tau \E_{\varepsilon\sim\NN{}}\left[\log p_{\sigma_k}(x + \nu_k\varepsilon)\right]\,,\\
    \bar{g}_k:x\in\R^d&\mapsto\E_{\varepsilon\sim\NN{}}\left[\log p_{\sigma_k}(x + \nu_k\varepsilon)\right].
\end{align}
Properties of $F_k$ were already given in \Cref{prop:conditionning}, and $\xstar[k]$ denotes again its single minimizer. As justified in the proof of \Cref{thm:renoised_iterates_cv}, $G_k$ and $\bar{g}_k$ are respectively twice and thrice continuously differentiable, and $G_k$ inherits $\lambda_k-$strong convexity and $L_k-$smoothness from $F_k$, with $L_k=L_f+\lambda_k +\frac{\tau}{\sigma_k^2}$.
We want to bound the sequence $\tilde{\beta}_k:=\E\lnorm{\tx_k - \xstar[k-1]}^2$. To do so, we  introduce $\beta_k:=\E_{k-1}\lnorm{\tx_k - \xstar[k-1]}^2$.
By definition, 
\begin{align}
& \, \, \tx_{k+1} - \xstar[k] \notag\\
    &=\tx_k - \alpha_k\left( \nabla f_{\sigma_k}(\tx_k)-\tau\nabla \log p_{\sigma_k}(\tx_k+\nu_k \varepsilon_k) \right) + (1 - \alpha_k)\nu_k\varepsilon_k - \xstar[k]\notag\\
    &=\left(\tx_k - \alpha_k \nabla G_k(\tx_k)  - \xstar[k]\right)+ \alpha_k\tau\left(\nabla \log p_{\sigma_k}(\tx_k+\nu_k \varepsilon_k)-\nabla\bar{g}_k(\tx_k) \right) + (1 - \alpha_k)\nu_k\varepsilon_k.
\end{align}
Taking the expectation over $\mathcal{F}_k$ of the squared norm, we obtain
\begin{align}
\beta_{k+1}=&\lnorm{\tx_k - \alpha_k \nabla G_k(\tx_k)  - \xstar[k]}^2 +\alpha_k^2\tau^2 \E_k\lnorm{\nabla \log p_{\sigma_k}(\tx_k+\nu_k \varepsilon_k)-\nabla\bar{g}_k(\tx_k) }^2\notag \\
&+ (1-\alpha_k)^2\nu_k^2d +2(1-\alpha_k)\alpha_k\nu_k \tau \E_k \langle \varepsilon_k, \nabla \log p_{\sigma_k}(\tx_k+\nu_k \varepsilon_k)-\nabla\bar{g}_k(\tx_k) \rangle .
\end{align}
We want to bound the sum on the right.

By adapting the calculations from \eqref{eq:perfect_oracle_gradient_critical_point} to \eqref{eq:simplified_renoised_bound} to a Gaussian smoothing with covariance $\nu_k^2\Id$ instead of $\sigma_k^2\Id$, and denoting $r_k:=\lnorm{\tx_k-\xstar[k]}$ and $h_k := \frac{\alpha_k\lambda_k L_k}{\lambda_k + L_k}$, we have $B>0$ such that:
\begin{equation}
\lnorm{\tx_k - \alpha_k \nabla G_k(\tx_k)  - \xstar[k]}^2 \leq \left(1 - h_k\right)r_k^2 + \left(3\alpha_k^2 + \dfrac{\alpha_k^2}{h_k}\right)B\nu_k^4.
\end{equation}
Again, since $ h_k=\frac{\alpha_k\lambda_kL_k}{\lambda_k+L_k}\geq \frac{\alpha_k\lambda_k}{2}$ and $1>\alpha_k\lambda_k$, we have $(3+\frac{1}{h_k})\leq \frac{5}{\alpha_k \lambda_k}$, and
\begin{equation}
\lnorm{\tx_k - \alpha_k \nabla G_k(\tx_k)  - \xstar[k]}^2 \leq \left(1 - h_k\right)r_k^2 + \frac{5\alpha_kB\nu_k^4}{\lambda_k}.
\end{equation}
Moreover, $\E_k\lnorm{\nabla \log p_{\sigma_k}(\tx_k+\nu_k \varepsilon_k)-\nabla\bar{g}_k(\tx_k) }^2$ corresponds to the conditional variance of $\nabla \log p_{\sigma_k}(\tx_k+\nu_k \varepsilon_k)$ with respect to $\mathcal{F}_k$. By the Poincaré inequality \citep{brascamp1976extensions}
\begin{equation}
    \E_k\lnorm{\nabla \log p_{\sigma_k}(\tx_k+\nu_k \varepsilon_k)-\nabla\bar{g}_k(\tx_k) }^2 \leq \nu_k^2  \E_k \lnorm{\nabla^2 \log p_{\sigma_k}(\tx_k + \nu_k \varepsilon_k)}^2_F.
\end{equation}
As the proof of \citet{pesme2025map} shows that, for any $\sigma>0$, $\lnorm{\nabla^3\log p_\sigma}_F\leq M$, 
\begin{align}
    &\E_k \lnorm{\nabla^2 \log p_{\sigma_k}(\tx_k + \nu_k \varepsilon_k)-\nabla^2 \log p_{\sigma_k}(\xstar[k])}^2_F\notag\\
    =\,&\E_k \lnorm{\int_0^1 \nabla^3\log p_{\sigma_k}((1-s)\xstar[k]+s(\tx_k + \nu_k \varepsilon_k))\,\dd s \, (\tx_k + \nu_k \varepsilon_k - \xstar[k])}_F^2 \notag\\
    \leq\,& \E_k \int_0^1 \lnorm{\nabla^3\log p_{\sigma_k}((1-s)\xstar[k]+s(\tx_k + \nu_k \varepsilon_k)) \, (\tx_k + \nu_k \varepsilon_k - \xstar[k])}_F^2\,\dd s\notag\\
    \leq\,& M^2 \E_k \lnorm{\tx_k + \nu_k \varepsilon_k -\xstar[k]}^2 \;=\;M^2(\nu_k^2d + r_k^2).
\end{align}
Thus, by \eqref{eq:ineqb} (which is also valid for the Frobenius norm) with $\delta=1$, we obtain
\begin{equation}
    \nu_k^2  \E_k \lnorm{\nabla^2 \log p_{\sigma_k}(\tx_k + \nu_k \varepsilon_k)}^2_F\leq 2\nu_k^2\left(M^2\nu_k^2d + M^2r_k^2) + \lnorm{\nabla^2 \log p_\sigma(\xstar[k])}_F^2\right).
\end{equation}
The sequence $\nu_k$ is clearly bounded, and by continuity of $(\sigma,x)\mapsto \nabla^2 \log p_\sigma(x)$ and boundedness of $(\xstar)_{\sigma\in(0,\sigma_0]}$ (\Cref{prop:lim_minim}), $(\lnorm{\nabla^2 \log p_\sigma(\xstar[k])}_F^2)_k$ is bounded. Therefore, there is $C>0$ such that, for any $k\geq 0$,
\begin{equation}
       \nu_k^2  \E_k \lnorm{\nabla^2 \log p_{\sigma_k}(\tx_k + \nu_k \varepsilon_k)}^2_F\leq \nu_k^2 C(1+r_k^2).
\end{equation}
Then, combining Cauchy-Schwarz and the previous inequality,
\begin{align}
    &|\E_k \langle \varepsilon_k, \nabla \log p_{\sigma_k}(\tx_k+\nu_k \varepsilon_k)-\nabla \bar{g}_k(\tx_k)\rangle| \nonumber \\ & \qquad \qquad\leq \sqrt{\E_k \lnorm{\varepsilon_k}^2\;    \E_k\lnorm{\nabla \log p_{\sigma_k}(\tx_k+\nu_k \varepsilon_k)-\nabla\bar{g}_k(\tx_k) }^2 }\notag\\
    &\qquad \qquad \leq \nu_k\sqrt{dC(1+r_k^2)}\leq\nu_k\sqrt{dC}(1+r_k^2/2).
\end{align}
Putting all together,
\begin{align}
    \beta_{k+1}\leq&\left(1-h_k+\alpha_k^2\nu_k^2\tau^2 C + (1-\alpha_k)\alpha_k \nu_k^2\tau \sqrt{dC}\right)r_k^2 \notag \\
    &+\left(3\alpha_k^2+ \dfrac{\alpha_k^2}{h_k}\right)B\nu_k^4 +(1-\alpha_k)^2\nu_k^2 d + 2(1-\alpha_k)\nu_k^2\alpha_k \tau \sqrt{dC}+C\alpha_k^2\tau^2\nu_k^2
\end{align}
Since $h_k \underset{k\to +\infty}{\sim}\alpha_k\lambda_k\underset{k\to +\infty}{\sim}\frac{\sigma_k^2\lambda_k}{\tau}$ and $\nu_k^2 = o(\sigma_k^2\lambda_k)$, up to considering a later iteration, we suppose that $\alpha_k^2\nu_k^2\tau C + (1-\alpha_k)\alpha_k \nu_k^2\tau \sqrt{dC}\leq \frac{h_k}{2}$, so that
\begin{align}
    \beta_{k+1}\leq&\left(1-\frac{h_k}{2}\right)r_k^2 \notag
    +\frac{5\alpha_kB}{\lambda_k}\nu_k^4 +(1-\alpha_k)^2\nu_k^2 d + 2(1-\alpha_k)\nu_k^2\alpha_k \tau \sqrt{dC}+C\alpha_k^2\tau^2\nu_k^2.
\end{align}

In order to obtain a recursive relation on $\tilde{\beta}_{k + 1}=\E[\beta_{k+1}]$, we need to make $\lnorm{\tx_k - \xstar[k-1]}^2$ appear on the right. To do so, we use \eqref{eq:ineqb} on $r_k^2$ with $\delta = \dfrac{h_k}{4 - 2h_k}$, which gives,

\begin{align}
    \beta_{k+1} &\leq \left(1 - \dfrac{h_k}{4}\right)\lnorm{\tx_k - \xstar[k-1]}^2 + \left(1 - \dfrac{h_k}{2}\right)\left(\dfrac{4}{h_k} - 1\right)\lnorm{\xstar[k]-\xstar[k-1]}^2 \notag\\
    &\quad +\nu_k^2\left(\frac{5\alpha_kB}{\lambda_k}\nu_k^2 +(1-\alpha_k)^2 d + 2(1-\alpha_k)\alpha_k \tau \sqrt{dC}+C\alpha_k^2\tau^2\right)\\
    &\leq \left(1 - \dfrac{h_k}{4}\right)\lnorm{\tx_k - \xstar[k-1]}^2 \notag\\
    &\quad +\underbrace{\dfrac{4}{h_k}\lnorm{\xstar[k]-\xstar[k-1]}^2 +\nu_k^2\left(\frac{5\alpha_kB}{\lambda_k}\nu_k^2 +(1-\alpha_k)^2 d + 2(1-\alpha_k)\alpha_k \tau \sqrt{dC}+C\alpha_k^2\tau^2\right)}_{:=\frac{h_k}{4}\delta_k}.\\
\end{align}

Hence, taking the total expectation, this yields
\begin{equation}
    \tilde{\beta}_{k + 1} \leq \left(1 - \frac{h_k}{4}\right)\tilde{\beta}_k + \dfrac{h_k}{4}\delta_k 
\end{equation}
Applying \Cref{lem:suitecv}, we obtain the convergence $\tilde{\beta}_k \underset{k\to +\infty}{\longrightarrow}\to 0$ under the following constraints:
\begin{equation}
    \sum_{k = 0}^{+\infty}h_k = +\infty,\quad \delta_k \underset{k\to+\infty}{\longrightarrow} 0
\end{equation}
Since $h_k\sim \dfrac{\sigma_k^2}{\tau}\lambda_k$, we recover condition \emph{(i)} of \Cref{thm:iterates_cv} from $\sum_{k = 0}^{+\infty}h_k = +\infty$. Then, by definition,
\begin{equation}
    \delta_k = \dfrac{16}{h_k^2}\lnorm{\xstar[k]-\xstar[k-1]}^2 + \dfrac{4}{h_k}\left(\nu_k^2\left(\frac{5\alpha_kB}{\lambda_k}\nu_k^2 +(1-\alpha_k)^2 d + 2(1-\alpha_k)\alpha_k \tau \sqrt{dC}+C\alpha_k^2\tau^2\right)\right).
\end{equation}
By the proof of \Cref{thm:renoised_iterates_cv}, we know that $\dfrac{16}{h_k^2}\lnorm{\xstar[k]-\xstar[k-1]}^2\underset{k\to +\infty}{\longrightarrow}0$ under $(ii)$ and $(iii)$. Moreover, since $\frac{\sigma_k^2}{\lambda_k}\underset{k\to+\infty}{\longrightarrow}0$,
\begin{equation}
    \dfrac{4\nu_k^2}{h_k}\left(\frac{5\alpha_kB}{\lambda_k}\nu_k^2 +(1-\alpha_k)^2 d + 2(1-\alpha_k)\alpha_k \tau \sqrt{dC}+C\alpha_k^2\tau^2\right)\underset{k\to+\infty}{\sim}\dfrac{4\tau\nu_k^2 d}{\sigma_k^2 \lambda_k}
\end{equation}
As $\nu_k^2=o(\sigma_k^2\lambda_k)$ by assumption, we obtain the convergence to $0$ of $\E\lnorm{\tx_{k + 1} - \xstar[k]}$. Using \Cref{prop:lim_minim} concludes the proof.

\end{proof}

\section{Discussion on the methods}
\subsection{Rewriting PnP-Flow}\label[appendix]{ssec:pnpflow}
The original PnP-Flow algorithm is defined as
\begin{equation}
    \label{eq:original_pnpflow}
    \begin{cases}
        z_{k + 1} = x_k - \alpha_k\nabla f(x_k)\,, \\
        \tilde z_{k + 1} = t_k z_{k + 1} + (1 - t_k)\varepsilon_k \, ,\quad \varepsilon_k\sim\NN{}\,, \\
        x_{k + 1} = \tilde D_{t_k}(\tilde z_{k + 1})\,.
    \end{cases}
\end{equation}

PnP-Flow uses the denoiser $\tilde D_t = \Id + (1 - t)v_t$ with $v:[0,1]\times \R^d\to \R^d$  a Flow Matching velocity field, defined as
\begin{equation}
    \forall (t,x) \in [0,1]\times \R^d, \; v_t(x)= \E[X-Z|(1-t)Z+tX = x],
\end{equation}
where $X\sim p$ and $Z\sim \NN{}$, taken independent. 
One can easily show that $\tilde D_t(x) = \E[X\vert(1-t)Z+tX=x]$. 

Defining $D_{\sigma}(x) = \E\left[X \vert  X + \sigma Z = x\right] = \mmse[](x)$, one has
\begin{equation}
    \tilde D_t(t x) = \E[X\vert(1-t)Z+tX=tx] = \E\left[X \vert \tfrac{1 - t}{t} Z + X = x\right] = D_{\sigma_{t}}(x) \, ,
\end{equation}
with $\sigma_t = \frac{1 - t}{t}$.  
Since $\tilde z_{k+1} = t_k(z_{k+1} + \sigma_{t_k} \varepsilon_k)$, $\tilde D_{t_k}(\tilde z_{k+1}) = D_{\sigma_{t_k}}(z_{k+1} + \sigma_{t_k} \varepsilon_k)$ and the last two steps of \eqref{eq:original_pnpflow} can be merged into a single noising/denoising step applied to $z_{k+1}$:%
\begin{equation}
    \label{eq:denoiser_pnpflow}
    \begin{cases}
        z_{k + 1} = x_k - \alpha_k\nabla f(x_k)\,, \\
        x_{k + 1} = D_{\sigma_{t_k}}(z_{k + 1} + \sigma_{t_k}\varepsilon_k), \quad \varepsilon_k\sim\NN{}\,.
    \end{cases}
\end{equation}

\textbf{Equivalence with MMSE Average}:
In the denoising case $f = \frac{1}{2}\lnorm{ \cdot - y}^2$ and based on the above discussion we have %
\begin{align}
    z_{k + 1} 
    &= x_k - \alpha_k (x_k - y)
    = (1 - \alpha_k) \mmse[k](z_k + \sigma_k \varepsilon_k) + \alpha_k y 
\end{align}
Hence, the $z_k$ iterates of PnP-Flow in that case correspond to iterations of MMSE Average, with the modification that some noise in added to the input of the denoiser.

\textbf{General case}:
Using a generic datafit $f$, PnP-Flow rewrites as two consecutive gradient steps, on $f$ and $- \sigma_k^2 \log p_{\sigma_k}$ respectively:
\begin{align}
    \begin{cases}
    z_{k+1} = x_k - \alpha_k \nabla f(x_k) \\
    x_{k+1} = z_k + \sigma_k \varepsilon_k + \sigma_k^2 \nabla \log p_\sigma(z_k  + \sigma_k \varepsilon_k)
    \end{cases}
\end{align}
The gradients of $f$ and $\log p_\sigma$ being evaluated at different points, we cannot hope to cast it as an instance of MMSE Averaging. 
Rather, the \genmmse{} we propose is similar in spirit, but writes as a single gradient step of $f$ and $\log p_\sigma$ simultaneously, closer to RED in that sense (but with annealing).

\section{Experiments}
\label{app:expes}
\paragraph{Initialisation. } When initializing \genmmse{}, we find that choosing $x_0 \approx \E_{X\sim p}\left[X\right]$, the mean natural image of the distribution, results in higher perceptual quality results. In practice, we approximate this mean image with $x_0 = \mmse[\mathrm{sample}](\varepsilon)$ where $\varepsilon\sim\NN{}$ and $\sigma_{\mathrm{sample}} = 100$.

\paragraph{Models}
For the two datasets,  Celeba-128 and AHFQ-256, we trained a Flow Matching model from scratch. We followed the architecture and training pipeline of \citet{martin2025pnp}, with one modification: we used an independent coupling instead of the minibatch OT coupling employed in their work. 

\subsection{Optimal hyper-parameters values}
\paragraph{Schedules. } For \genmmse{}, we take the schedules from \Cref{rem:conv_schedules}, which are proven to make the scheme converge.
\begin{equation}
    \lambda_k = \dfrac{\lambda_0}{(k + 1)^\beta}, \quad \sigma_k^2=\dfrac{\sigma_0^2}{(k + 1)^\gamma},\quad \alpha_k = \dfrac{\sigma_k^2}{\sigma_k^2 + \kappa\tau}\,,
\end{equation}
with $0 < \beta < \gamma < 1$, $\beta + \gamma < 1$, $\lambda_0\in(0,1)$. We choose to tune the noise level $\kappa\tau := \kappa \sigma_y^2$ with a factor $\kappa$. We also tune the parameters $\lambda_0, \sigma_0$.
For Renoising-GAMMA, we approximate $\bbE_\epsilon[\mmse]$ using $N_\epsilon =3$ draws of $\epsilon$ per iteration, since \Cref{fig:ablation_num_eps} shows no significant differences when using more draws. 
We set $N = 2000$,$ \gamma = 0.98$ $, \beta=0.01 $, for GAMMA (noiseless) and Renoised-GAMMA, the resulting noise schedules are displayed in \Cref{fig:sigmas_k}.

For PnP-Flow, we take the same hyperparameters values as the ones used in the original paper~\citep{martin2025pnp}. 

\begin{table}[H]
	\centering
	\caption{Hyper-parameters used and \genmmse{} on CelebA and AFHQ-Cat datasets.}
	\resizebox{\textwidth}{!}{
		\begin{tabular}{llccccc}
			\toprule
			& & \textbf{Denoising} & \textbf{Deblurring} & \textbf{Super-res.} & \textbf{Rand. inpaint.} & \textbf{Box inpaint.} \\
			\midrule
			\multirow{2}{*}{\textbf{CelebA}}
			& \multicolumn{6}{l}{\textbf{PnP-Flow}} \\
			& $N$ (number of steps)    & 100 & 100    & 100    & 100    & 100    \\
			& $\alpha$ (learning rate) & 0.8 & 0.01   & 0.3    & 0.01   & 0.5    \\\cmidrule{2-7}
            & \multicolumn{6}{l}{\textbf{Approx-PGD}} \\
            & $\gamma$ (step size)    & 1.0 & 1.0    & 1.0    & 1.0    & 1.0    \\
			& $\kappa$ (noise level factor) & 1.0 & 0.1   & 0.5    & 2.0   & 0.5    \\\cmidrule{2-7}
            \multirow{4}{*}{\shortstack[l]{$ \gamma = 0.98$ \\$ \beta=0.01 $ \\ $ N = 2000$ \\ $\lambda_0 = 10^{-3}$}}& \multicolumn{6}{l}{\textbf{GAMMA}} \\
            & $\sigma_0$ & 1.0 & 1 & 5 & 1 & 10 \\
            & $\kappa$ & 2.0 & 0.1 & 0.5 & 50 & 20 \\
			& \multicolumn{6}{l}{\textbf{Renoised-GAMMA}} \\
            & $\sigma_0$ & 1.0 & 10 & 5 & 5 & 10 \\
            & $\kappa$ & 2.0 & 5 & 5 & 50 & 20 \\
			\midrule
			\multirow{2}{*}{ \textbf{AFHQ}  }
			& \multicolumn{6}{l}{\textbf{PnP-Flow}} \\
			& $N$ (number of steps)    & 100 & 100    & 500    & 200    & 100    \\
			& $\alpha$ (learning rate) & 0.8 & 0.01   & 0.01   & 0.01   & 0.5    \\\cmidrule{2-7}
            & \multicolumn{6}{l}{\textbf{Approx-PGD}} \\
            & $\gamma$ (step size)    & 1.0 & 1.0    & 1.0    & 1.0    & 1.0    \\
			& $\kappa$ (noise level factor) & 1.0 & 2.0   & 1.0    & 5.0   & 2.0    \\\cmidrule{2-7}
            \multirow{3}{*}{\shortstack[l]{$ \gamma = 0.98$ \\$ \beta=0.01 $ \\ $ N = 2000$}} & \multicolumn{6}{l}{\textbf{GAMMA}} \\
			& $\lambda_0$ (number of steps)    & $10^{-3}$ & $10^{-2}$     & $10^{-3}$ & $10^{-3}$    & $10^{-3}$ \\
            & $\sigma_0$ & 5 & 10 & 10& 1 & 5 \\
            & $\kappa$ & 1 & 1 & 20 & 50 & 50 \\
			& \multicolumn{6}{l}{\textbf{Renoised-GAMMA}} \\
			& $\lambda_0$ (number of steps)    & $10^{-3}$ & $10^{-3}$     & $10^{-3}$    & $10^{-3}$    & $10^{-3}$    \\
            & $\sigma_0$ & 1 & 10 & 5& 5 & 10 \\
            & $\kappa$ & 1 & 10 & 5 & 50 & 20 \\
			\bottomrule
		\end{tabular}
        }
	\label{tab:hyperparameters}
\end{table}

\subsubsection{Relaxed \genmmse{} schedules}
Given a maximum number of iterations $N$, we consider the following family of schedules,
\begin{equation}
    \sigma_k^2 = \left(\dfrac{N - k}{k}\right)^2, \quad \lambda_k=\lambda_0\sigma_k^\beta,\quad \alpha_k = \dfrac{\sigma_k^2}{\kappa\sigma_k^2 + \tau}
\end{equation}
where $\lambda_0 > 0$ and $\kappa \geq 1$.

\begin{table}[H]
    \centering
    \caption{Hyperparameters used for \genmmse{} on the CelebA and AFHQ-Cat datasets.}
    \resizebox{\textwidth}{!}{
        \begin{tabular}{llccccc}
            \toprule
            & & \textbf{Denoising} & \textbf{Deblurring} & \textbf{Super-res.}
            & \textbf{Rand. inpaint.} & \textbf{Box inpaint.} \\
            \midrule
            \multirow{5}{*}{\textbf{CelebA}}
            & \multicolumn{6}{l}{\genmmse{}} \\
            & $N$ (number of steps) & 100 & 100 & 500 & 500 & 100 \\
            & $\kappa$              & 1.0 & 1.0 & 2.0 & 1.5 & 1.2 \\
            & $\lambda_0$           & $10^{-3}$ & $10^{-3}$ & $10^{-3}$ & $10^{-3}$ & $10^{-3}$ \\
            & $\beta$               & 1.0 & 1.0 & 1.0 & 1.0 & 1.0 \\
            \midrule
            \multirow{5}{*}{\textbf{AFHQ-Cat}}
            & \multicolumn{6}{l}{\genmmse{}} \\
            & $N$ (number of steps) & 100 & 100 & 500 & 500 & 100 \\
            & $\kappa$              & 1.0 & 1.0 & 2.0 & 1.5 & 1.2 \\
            & $\lambda_0$           & $10^{-3}$ & $10^{-3}$ & $10^{-3}$ & $10^{-3}$ & $10^{-3}$ \\
            & $\beta$               & 1.0 & 1.0 & 1.0 & 1.0 & 1.0 \\
            \bottomrule
        \end{tabular}
    }
    \label{tab:hyperparameters_genmmse}
\end{table}

\begin{figure}[thb]
    \centering
    \includegraphics[width=0.7\linewidth]{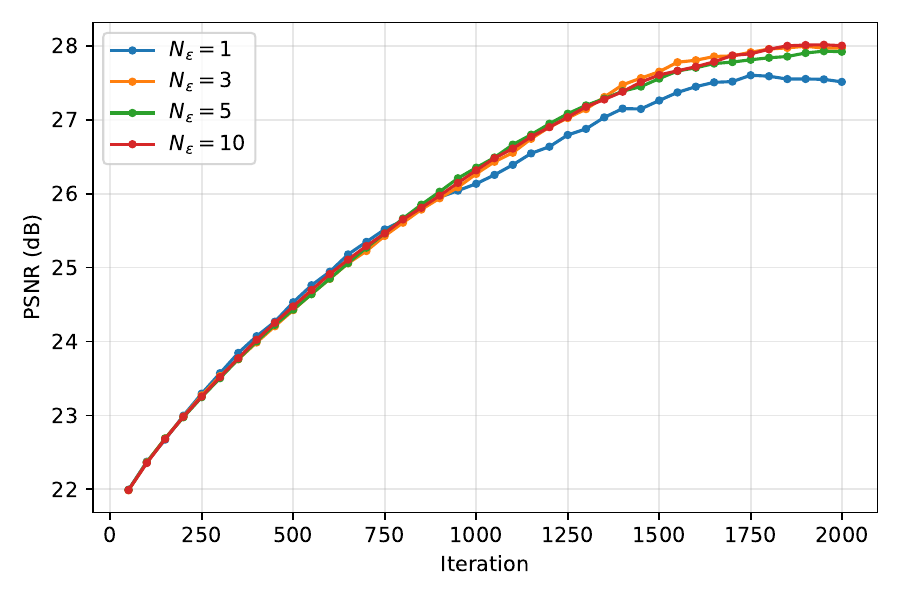}
    \caption{Ablation of the number $N_\epsilon$ of noise samples $\epsilon$ used to approximate $\bbE_{\epsilon} [ \mmse( x +\sigma \epsilon) ] $ - AFHQ - Box inpainting}
    \label{fig:ablation_num_eps}
\end{figure}

\begin{figure}
    \centering
    \includegraphics[width=0.7\linewidth]{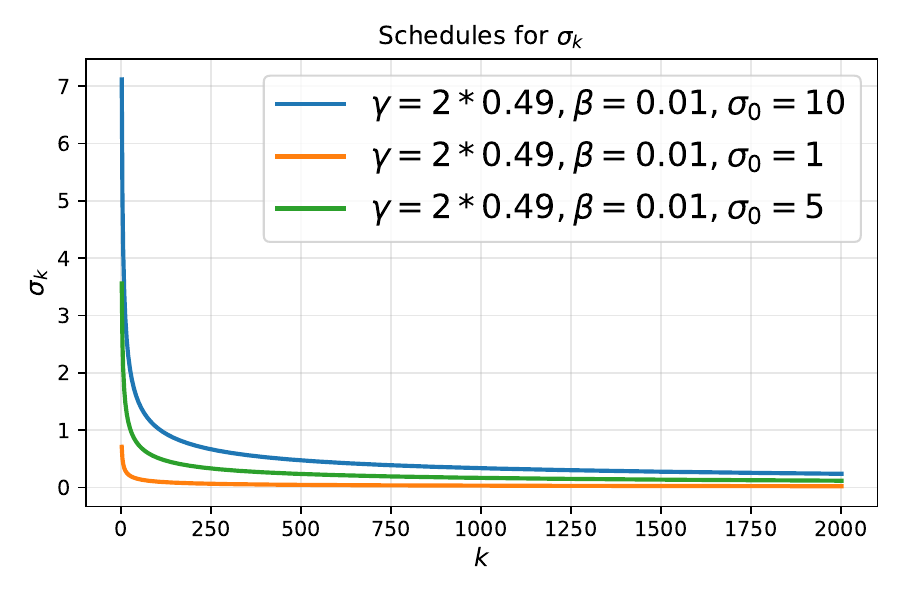}
    \caption{The two schedules $\sigma_k = \frac{\sigma_0}{(k+1)^{\gamma/2}}$ used in our experiments. Setting $N=2000$ across all experiments ensures that $\sigma_k$ is sufficiently close to $0$.}
    \label{fig:sigmas_k}
\end{figure}

\begin{table}[tbh]
\caption{Comparisons of methods on different inverse problems on the AFHQ dataset. Results are averaged across 100 test images. Higher PSNR / SSIM is better, lower LPIPS is better.}
\label{tab:benchmark_results_afhq}
\centering
\setlength{\tabcolsep}{2.8pt}
\renewcommand{\arraystretch}{1.25}
\resizebox{1.0\hsize}{!}{
\begin{tabular}{l c|ccc|ccc|ccc|ccc|ccc}
\toprule
\multirow{3}{*}{Method}
& \multirow{3}{*}{{\small Theory}}
& \multicolumn{3}{c|}{\textcolor{blue}{Denoising}}
& \multicolumn{3}{c|}{\textcolor{blue}{Deblurring}}
& \multicolumn{3}{c|}{\textcolor{blue}{Super-res.}}
& \multicolumn{3}{c|}{\textcolor{blue}{Rand. inpaint.}}
& \multicolumn{3}{c}{\textcolor{blue}{Box inpaint.}} \\

& & \multicolumn{3}{c|}{\textcolor{blue}{\small $\sigma=0.2$}}
& \multicolumn{3}{c|}{\textcolor{blue}{\small $\sigma=0.05$, $\sigma_{\mathrm b}=3.0$}}
& \multicolumn{3}{c|}{\textcolor{blue}{\small $\sigma=0.05$, $\times4$}}
& \multicolumn{3}{c|}{\textcolor{blue}{\small $\sigma=0.01$, $70\%$}}
& \multicolumn{3}{c}{\textcolor{blue}{\small $\sigma=0.05$, $80\times80$}} \\

\cmidrule(lr){3-5}
\cmidrule(lr){6-8}
\cmidrule(lr){9-11}
\cmidrule(lr){12-14}
\cmidrule(lr){15-17}

& & \small PSNR & \small SSIM & \small LPIPS
& \small PSNR & \small SSIM & \small LPIPS
& \small PSNR & \small SSIM & \small LPIPS
& \small PSNR & \small SSIM & \small LPIPS
& \small PSNR & \small SSIM & \small LPIPS \\

\midrule

\rowcolor{gray!10}
Degraded 
 &  & 
20.00 & 0.292 & 0.550 &
24.43 & 0.531 & 0.452 &
11.77 &0.220 &0.877 &
 13.44 & 0.221 & 1.090&
21.68 & 0.726 & 0.219 \\

\hdashline

\shortstack[l]{PnP-Flow}
& No & 
\underline{32.27} & \underline{0.876} & 0.164 & 
\textbf{29.48} & \textbf{0.795} & 0.317 &
\textbf{29.02} & \textbf{0.811} & \textbf{0.171} &
\textbf{34.97} & \textbf{0.939}& \textbf{0.037}&
\textbf{28.63} & \underline{0.911} &\underline{0.107}  \\

\shortstack[l]{Approx-PGD}
& Yes &
28.21 & 0.679 & 0.236 &
18.42 & 0.464 & 0.526 &
7.08 & 0.059 & 0.690 &
18.36& 0.404 & 0.510 &
17.86 & 0.717 & 0.354 \\

\hdashline

\rowcolor{green!5!violet!10}
\shortstack[l]{GAMMA\\(noiseless)}
& Yes&
30.03 & 0.751 & \textbf{0.102} &
14.23 & 0.311 & 0.613 &
26.60 & 0.703 & 0.361 &
22.09 &  0.593 & 0.326 &
21.91 & 0.603 & 0.277 \\

\rowcolor{green!5!violet!10}
\shortstack[l]{GAMMA\\(noise)}
& Yes &
30.74 & 0.848 & 0.212 &
28.89 & 0.779 & 0.315 &
25.29 & 0.721&0.279  &
32.56 & 0.895 & 0.090 &
26.21 & 0.852 & 0.101 \\

\rowcolor{green!5!violet!10}
\shortstack[l]{GAMMA\\(relaxed)}
& No &
\textbf{32.46} & \textbf{0.889} & \underline{0.104} &
\underline{29.10} & \underline{0.789}& \textbf{0.292} &
\underline{28.18} & \underline{0.782}  & \underline{0.259} &
\underline{34.40} & \underline{0.931} & \underline{0.046}  &
\underline{27.72} & \textbf{0.919} & \textbf{0.069} \\

\bottomrule
\end{tabular}
}
\end{table}

\begin{figure}[htp]
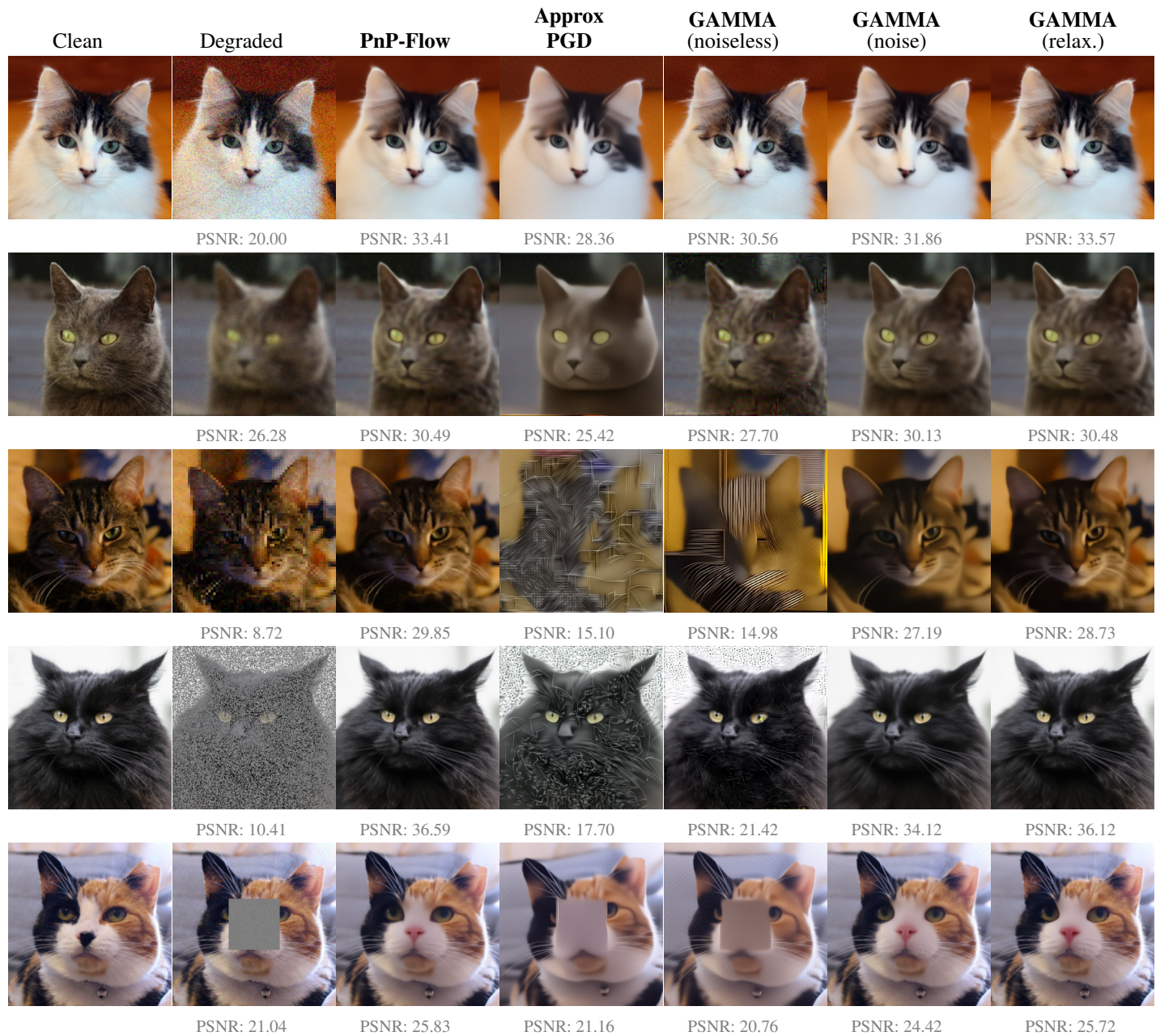

    \centering
    \begin{tabular}{cccccccc}

    Clean &
    Degraded &
    \textbf{PnP-Flow} &
    \shortstack{\textbf{Approx} \\ \textbf{PGD}} &
    \shortstack{\textbf{GAMMA} \\ (noiseless)} &
    \shortstack{\textbf{GAMMA} \\ (noise)} &
    \shortstack{\textbf{GAMMA} \\(relax.)}
    \\

    \includeproblemimagesafhq{denoising}{0}{0}{20.00}{33.41}{28.36}{30.56}{31.86}{33.57} \\
    \includeproblemimagesafhq{gaussian_deblurring_FFT}{1}{1}{26.28}{30.49}{25.42}{27.70}{30.13}{30.48}\\
    \includeproblemimagesafhq{superresolution}{0}{1}{8.72}{29.85}{15.10}{14.98}{27.19}{28.73} \\
    \includeproblemimagesafhq{random_inpainting}{1}{2}{10.41}{36.59}{17.70}{21.42}{34.12}{36.12} \\
    \includeproblemimagesafhq{inpainting}{1}{3}{21.04}{25.83}{21.16}{20.76}{24.42}{25.72}

    \end{tabular}
    \caption{Qualitative results on AFHQ-256. }
    \label{fig:afhq_images_full}
\end{figure}

\end{document}